\documentclass[conference]{IEEEtran}
\IEEEoverridecommandlockouts
\usepackage{cite}
\usepackage{amsmath,amssymb,amsfonts}
\usepackage{amsthm}
\usepackage{mathtools}
\usepackage{nccmath}
\usepackage[noend]{algorithmic}
\usepackage{graphicx}
\usepackage{textcomp}
\usepackage{xcolor}
\def\BibTeX{{\rm B\kern-.05em{\sc i\kern-.025em b}\kern-.08em
    T\kern-.1667em\lower.7ex\hbox{E}\kern-.125emX}}
\usepackage[hidelinks]{hyperref}
\usepackage{stfloats}

\usepackage{algorithm}
\usepackage{float}
\usepackage{setspace}
\usepackage{subfigure}
\usepackage{verbatim}
\usepackage[misc]{ifsym}
\usepackage{enumitem}
\usepackage{caption}
\usepackage{extpfeil}
\usepackage{ifsym}
\usepackage{enumitem}
\allowdisplaybreaks[4]

\newtheorem{problem}{Problem}
\newtheorem{definition}{Definition}

\newtheorem{proposition}{Proposition}
\def\BibTeX{{\rm B\kern-.05em{\sc i\kern-.025em b}\kern-.08em
    T\kern-.1667em\lower.7ex\hbox{E}\kern-.125emX}}

\makeatletter
\def\ps@IEEEtitlepagestyle{%
	\def\@oddfoot{\mycopyrightnotice}%
	\def\@evenfoot{}%
}
\def\mycopyrightnotice{%
	{\footnotesize 978-1-6654-8045-1/22/\$31.00~\copyright2022 IEEE \hfill} % Revise this line accordingly!
	\gdef\mycopyrightnotice{}
}

\newif\ifextendedreport

\newcommand{\showappendix}[1]{%
	\ifextendedreport
	\clearpage
	\appendices
	#1%
	\else
	\fi
}

\newcommand{\venueforappendix}[1]{%
	\ifextendedreport
	Appendix #1%
	\else
	our extended report \cite{zheng2021fl}%
	\fi
}

\begin{document}

\extendedreporttrue

\title{\huge{FL-Market}: Trading Private Models in Federated Learning}

\author{\IEEEauthorblockN{Shuyuan Zheng\IEEEauthorrefmark{1},
Yang Cao\IEEEauthorrefmark{2}\textsuperscript{\Letter}, Masatoshi Yoshikawa\IEEEauthorrefmark{1},
Huizhong Li\IEEEauthorrefmark{3}, Qiang Yan\IEEEauthorrefmark{4}}
\IEEEauthorblockA{\IEEEauthorrefmark{1}Kyoto University,
\IEEEauthorrefmark{2}Hokkaido University,
\IEEEauthorrefmark{3}WeBank Co., Ltd.,
\IEEEauthorrefmark{4}Singapore Management University\\
Email: \IEEEauthorrefmark{1}\{caryzheng@db.soc., yoshikawa@\}i.kyoto-u.ac.jp,
\IEEEauthorrefmark{2}yang@ist.hokudai.ac.jp,\\
\IEEEauthorrefmark{3}wheatli@webank.com,
\IEEEauthorrefmark{4}qiang.yan.2008@smu.edu.sg}}

\maketitle

\begin{abstract}
Acquiring a sufficient amount of training data is a significant bottleneck for machine learning (ML) based data analytics.
Recently, commoditizing ML models has been proposed as an economical and moderate solution to ML-oriented data acquisition.
%However, existing model marketplaces assume that the broker is trusted, and they thus authorize him to access and even control data owners' private training data, which is not realistic in practice.
However, existing model marketplaces assume that the broker can access data owners' private training data, which may not be realistic in practice.
In this paper, to promote trustworthy data acquisition for ML tasks, we propose FL-Market, a \textit{locally} private model marketplace that protects privacy against not only model buyers but also an untrusted broker.
FL-Market decouples ML from the need to centrally gather training data on the broker's side using \textit{federated learning}, a privacy-preserving ML paradigm in which data owners collaboratively train an ML model by uploading local gradients (to be aggregated into a global gradient for model updating).
Then, FL-Market enables data owners to locally perturb their gradients by \textit{local differential privacy} and thus further prevents privacy risks. 
To drive FL-Market, we propose a deep learning-empowered auction mechanism for intelligently deciding the local gradients' perturbation levels and an optimal aggregation mechanism for aggregating the perturbed gradients.  
Our auction and aggregation mechanisms can jointly maximize the global gradient's accuracy, which optimizes model buyers' utility.
Our experiments verify the effectiveness of the proposed mechanisms.
\end{abstract}


\section{Introduction}
Machine learning (ML) based data analytics has demonstrated great success in many domains.
Acquiring a sufficient amount of private data to train ML models usually needs considerable expenses, especially as data owners are becoming increasingly aware of the value of their data and the severe risks from uncontrolled data usage after sharing the data.
Consequently, recent efforts have proposed \textit{model marketplaces} \cite{chen2019towards, jia2019efficient, agarwal2019marketplace, liu2021dealer, jiang2022pricing} where a data broker commercializes data owners' private data in the form of ML models to facilitate ML-oriented data acquisition.
% that commoditize ML models as an economical and moderate solution to ML-oriented data acquisition.
Since model buyers do not contact training data directly, this category of business models can relieve data owners' concerns about losing control over their data and thus incentivize data sharing to some extent.

However, data owners still face notable privacy risks in the existing model marketplaces, which may make them hesitate to contribute data.
Although some works (e.g., \cite{chen2019towards, liu2021dealer, jiang2022pricing}) reduce privacy leakage to model buyers by injecting random noise into ML models using central differential privacy (CDP) \cite{dwork2006calibrating}, existing works assume that the broker is trusted and authorized to access and control the raw data.
This assumption is unrealistic, considering that many giant companies have been involved in user data breaches or privacy scandals.
Therefore, we demand a model marketplace that protects privacy against not only model buyers but also its broker.

\textit{Federated learning} (FL) \cite{mcmahan2016federated} has emerged as a promising paradigm for privacy-preserving ML.
% E.g., FL is recommended for complying with the data minimisation principle in GDPR, which requires data controllers to only process the minimum amount of personal data they need \cite{ico}.
Unlike traditional ML that requires training data to be stored on a centralized server (e.g., a broker in a model marketplace), FL enables the clients (i.e., data owners) to collaboratively train a model by uploading local updates (e.g., gradients) and, meanwhile, to keep their own training data on the local sides.
Since FL decouples ML from the need to centrally gather training data, it can largely restrict an untrusted server's ability to acquire private information.
Even though the local gradients trained on the raw data can be sensitive \cite{NEURIPS2019_60a6c400}, many works \cite{wang2020federated, sun2020ldp, liu2020fedsel, liu2020flame,wu2022fedctr} suggest that \textit{local differential privacy} (LDP) \cite{evfimievski2003limiting} can be combined with FL to perturb the gradients on the local sides and thus protect privacy.

\begin{figure}[ht]
    \centering
    \includegraphics[scale=0.3]{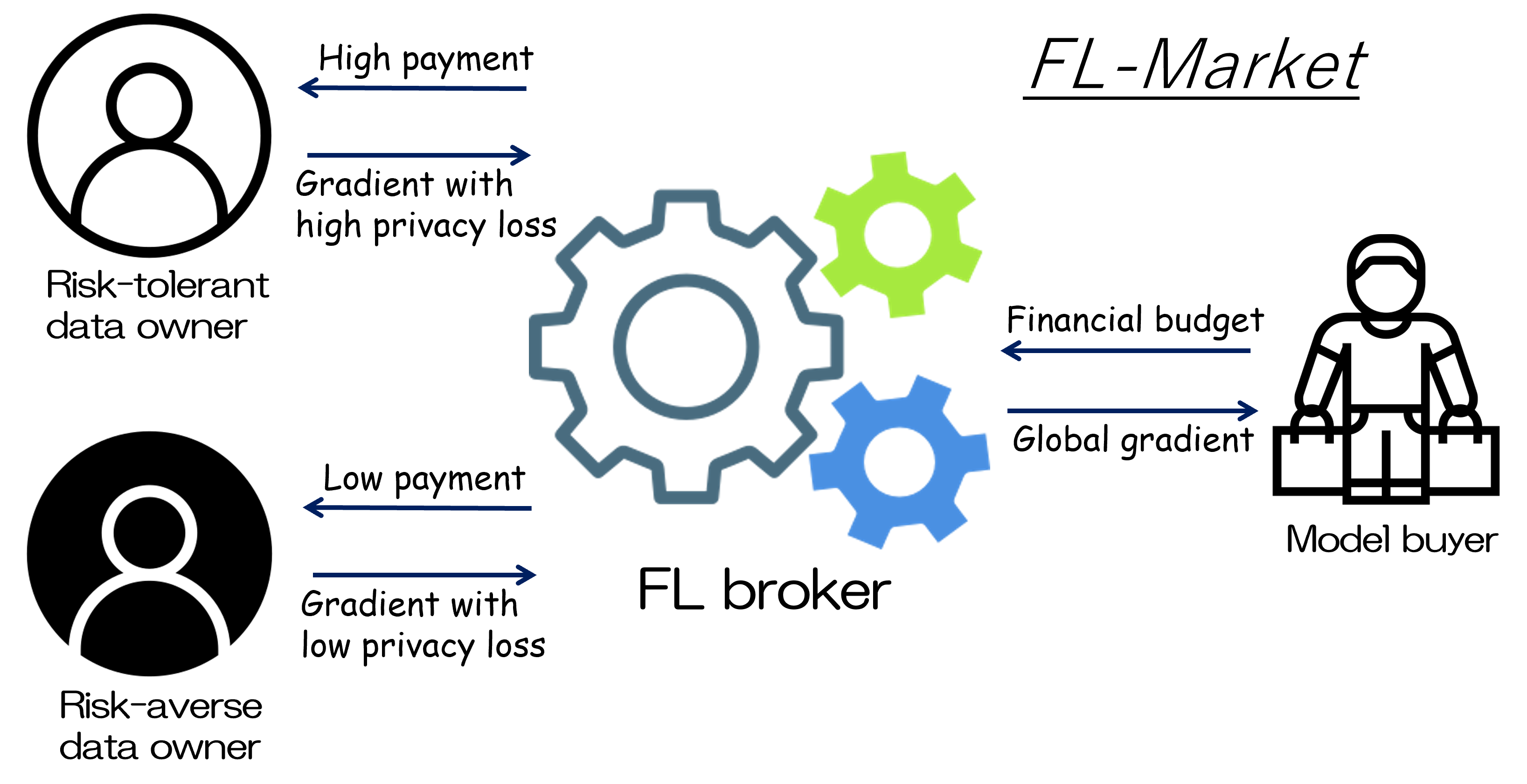}
    %  \vspace{-10pt}
    \caption{FL-Market allows data owners to control the perturbation level of their gradients in each round of FL training. Those data owners who contribute more accurate gradients (i.e., with less noise) will receive higher payments.}
%   \vspace{-10pt}
    \label{fig:fl_market}
\end{figure}

In this paper, for the first time, we propose a locally private model marketplace empowered by FL and LDP, called \textit{FL-Market} (Federated Learning Based Locally Private Model Market), to promote trustworthy data acquisition for ML-based data analytics.
% In this paper, for the first time, we \textit{simultaneously} address incentivization and privacy protection in FL with our proposed FL-Market.
Figure \ref{fig:fl_market} depicts the three parties in FL-Market: data owners, model buyers, and an FL broker.
%\footnote{In some cases, both FL broker and the buyers could be analysts who want to train machine learning models on data owners' data.}.
%We enable data owners to trade their gradients with specified privacy loss and obtain compensations according to their valuations on the privacy loss.
The FL broker coordinates FL-based model training and trading between data owners and model buyers.
A model buyer attempts to purchase ML models with a financial budget.
Data owners do not sell their raw data; instead, they sell locally private gradients perturbed by LDP in the training process coordinated by the FL broker.
The perturbation level is controlled by a privacy parameter $\epsilon$, which LDP formally defines as a metric of privacy loss.
% To incentivize contribution, we follow seminal differentially private data marketplaces \cite{roth2012buying, nissim2012privacy, fleischer2012approximately, nissim2014redrawing, ghosh2015selling, zhang2020selling} to employ an auction-based method for pricing gradients.
To incentivize contribution, we follow seminal differentially private data marketplaces \cite{ghosh2011selling, ghosh2015selling, roth2012buying, nissim2012privacy, fleischer2012approximately, nissim2014redrawing} to employ an auction-based method for pricing gradients.
Concretely, we allow each owner to report (bid) her valuation of privacy loss, named \textit{privacy valuation}, and report the maximum tolerable privacy loss, called \textit{privacy budget}.
Then, the broker uses an \textit{auction mechanism} to decide each owner's privacy parameter and compensate for the corresponding privacy loss according to her privacy valuation.
The auction should guarantee \textit{truthfulness}, which means each data owner (i.e., a bidder) will never obtain a higher utility by reporting an untruthful privacy valuation and budget.
%each data owner (bidder) will not obtain more profit if she strategically misreports her privacy valuation and budget.
% provided  she truthfully reports her privacy valuation and budget. 
Finally, the perturbed local gradients are aggregated into a global gradient by an \textit{aggregation mechanism} to update the buyer's model.
%We propose FL-Market, a privacy-preserving gradient market for federated learning. 
%Data owners report their \textit{privacy valuation} (i.e., a function of compensation over her privacy loss) with the maximum tolerable privacy loss.
%Data owners can obtain compensations according to their valuations on the privacy loss.
% Figure \ref{fig:fl_market} depicts our insight: By meeting data owners' dispersed privacy preferences and providing appropriate compensation, we can (1) incentivize privacy risk-tolerant data owners to set larger privacy parameters (that result in gradients with less noise) and (2) provide strict privacy protection for risk-averse data owners. 
%we can (1) obtain more quality gradients with high privacy loss from privacy risk-tolerant data owners, and (2) provide a higher privacy protection with low privacy loss for privacy risk-averse data owners. 
% In a nutshell, in this paper, we allow data owners to sell their gradients perturbed by LDP with personalized privacy protection levels, and obtain payments based on their valuations on those gradients.

Building this model marketplace calls for an elaborate mechanism design that enables the auction and aggregation mechanisms to jointly optimize the global gradient's utility.
% First, unlike the trusted broker in the previous works who collects raw data and then simply averages or adds up them into a statistic, in our setting, the untrusted broker has to aggregate the locally perturbed gradients with the consideration of their various accuracy levels.  
First, in FL-Market, the broker has to aggregate the locally private gradients considering their various accuracy levels.
Consequently, the aggregation mechanism should factor in the privacy losses decided by the auction mechanism when making a decision.
Second, the auction mechanism should properly purchase local gradients to maximize the aggregated gradient's utility, which implies that the aggregation decision feeds back into the auction decision. 
However, the aggregation mechanism may fail to provide an analytical solution.
In this case, the utility-maximizing objective of our auction problem also cannot be expressed in an analytic form, which makes it extremely challenging to characterize and design an optimal truthful mechanism. 
In a nutshell, the need for joint optimization dramatically increases the complexity of optimal mechanism design.

Our main contributions are threefold.
\begin{itemize}[leftmargin=*, topsep=0pt]
    \item We design a novel privacy-preserving model trading framework, \textit{FL-Market}, for acquiring locally private ML models via FL (Section \ref{sec:framework}). 
    % In FL-Market, both risk-tolerant and risk-averse data owners' dispersed privacy preferences are fully respected, and they are incentivized to contribute gradients protected by LDP.
    In FL-Market, data owners maintain control of their raw data by FL and enjoy the desired level of privacy against both the broker and model buyers using LDP.
    To the best of our knowledge, FL-Market is the first \textit{locally} private model marketplace.
    On the other end, we formulate optimization problems for designing the auction and aggregation mechanisms with the objective of maximizing the global gradient's accuracy, which optimizes model buyers' utility.
    \item We propose an optimal aggregation mechanism \textit{OptAggr} for FL with personalized LDP parameters (Section \ref{sec:ag}).
    The conventional practice of FL aggregates gradients with weights proportional to clients' data sizes (i.e., all samples are uniformly weighted), which may not be optimal when the gradients are perturbed to different extents.
    We transform the problem of designing an optimal aggregation mechanism under personalized privacy losses into an equivalent \textit{quadratic programming problem}.
    We prove that the equivalent problem is convex and thus can be solved by off-the-shelf optimizers.
    Supported by the optimizers, OptAggr decides the optimal way to aggregate the gradients.
    % \item We propose two aggregation mechanisms to compute accurate aggregated gradients for buyers (Section \ref{sec:ag}). 
    % To optimally aggregated gradients under personalized LDP, we formalize an error bound minimization problem. 
    % Then, we adapt the conventional gradient aggregation method of data size-weighted averaging \cite{mcmahan2016federated} as a baseline solution, which results in the optimal bias bound in the worst case.  
    % Further, we further propose an advanced solution \textit{VarOpt} that can minimize the aggregated gradients' variance by better leveraging personalized privacy parameters.
    \item We propose a novel auction mechanism, \textit{DM-RegretNet}, to incentivize data owners to contribute accurate gradients (Section \ref{sec:ac}).
    % We propose a novel deep learning-empowered auction mechanism, \textit{DM-RegretNet} (Section \ref{sec:ac}).
    % First, we make a practical assumption regarding data owners’ privacy valuations to relax our auction problem and propose a truthful auction mechanism, \textit{All-in}, which adopts a heuristic strategy to optimize the auction objective. 
    % We relax the truthfulness constraint and propose an approximately truthful auction mechanism: .
    Concretely, to design an optimal mechanism that jointly optimizes the gradient's utility with the aggregation mechanism, we seek support from RegretNet, the state-of-the-art deep learning-empowered automated mechanism design technique \cite{dutting2019optimal}.
    % To design a truthful mechanism for our \textit{budget-limited multi-unit multi-item procurement auction problem}, which remains severe in the literature \cite{chan2014truthful}, we seek support from RegretNet, the state-of-the-art automated auction design technique \cite{dutting2019optimal}.
    However, RegretNet always generates \textit{randomized} allocation results for auction items (i.e., the privacy losses in our case), which makes it tough to maximize the global gradient's accuracy. 
    On the contrary, DM-RegretNet (Deterministic Multi-Unit RegretNet) yields \textit{deterministic} auction decisions jointly with OptAggr and thus can significantly improve the global gradient's utility. 
    Our extensive experiments demonstrate that DM-RegretNet can achieve better model accuracy and approximate the truthfulness constraint more closely than RegretNet.

\end{itemize}

% \vspace{-5pt}
\section{Preliminary}
\label{sec:prelim}
% In this section, we introduce the training procedure for FL and some basic concepts in LDP.

\subsubsection*{Federated learning}
\label{sec:fl}
FL is a privacy-preserving framework for collaborative ML. 
In a typical FL architecture, $n$ data owners $\{1,...,n\}$ collaboratively train an ML model $h_{w}(\cdot)$ using their datasets $\{D_1,..., D_n\}$ under the coordination of an FL server (e.g., the FL broker in FL-Market), where $w$ is a set of model parameters.
The training process consists of multiple training rounds $1,..., R$.
We show a training round $r\in [R]$ of the widely-used FedSGD algorithm \cite{mcmahan2016federated} as follows.

\begin{enumerate}[leftmargin=*]
    \item \textbf{Model broadcasting}: The server broadcasts model parameters $w^r$ with a loss function $l(\cdot)$.
    \item \textbf{Local training}: Each data owner $i$ computes a local gradient $g_i$ using her local dataset $D_i=[r_{i,j}]_{j\in[d_i]}$ consisting of $d_i$ records.
    The gradient $g_i$ is the mean gradient of the records, i.e., $g_i=\mathbb{E}_{r \in D_i}[\nabla l(w^r; r)]$.
    \item \textbf{Gradients aggregation}: The server collects all the local gradients and aggregates them into a global gradient $g^*$ by averaging, i.e., $g^* = \sum_{i=1}^{n} \frac{d_i}{d_1+...+d_n} g_i$ where $d_i$ denotes the size of $D_i$.
    \item \textbf{Model updating}: The server updates the model parameters $w^r$ by the global gradient, i.e., $w^{r+1}=w^r-\eta \cdot g^*$ where $\eta\in R^+$ is a learning rate.
\end{enumerate}
In addition, gradient clipping is a widely used method for avoiding the exploding gradient problem \cite{bengio1994learning} where unacceptably large gradients make the training process unstable. 
In this paper,  we adopt the gradient clipping method $clip$ \cite{pascanu2013difficulty} that rescales a gradient $g_i$ if its norm cannot be covered by a threshold $L$, i.e., $clip(g_i, L)=g_i\cdot\min(1, \frac{L}{||g_i||_1})$.
To reduce notational overload, we let each $g_i$ denote the clipped version in the rest of this paper, i.e.,
\begin{equation}
\label{eq:gradient}
    g_i=\mathbb{E}_{r \in D_i}[\nabla l(w^r; r)]\cdot \min(1, \frac{L}{||\mathbb{E}_{r \in D_i}[\nabla l(w^r; r)]||_1})
\end{equation}

\subsubsection*{Local differential privacy}
\label{sec:ldp}
LDP \cite{evfimievski2003limiting} is a de facto data privacy definition.
In FL, even if data owners maintain their datasets on the local sides, their private information still can be inferred from the uploaded gradients by the server \cite{NEURIPS2019_60a6c400}. 
To prevent privacy leakage, data owners can use an LDP perturbation mechanism $\mathcal{M}$, such as the Laplace mechanism \cite{dwork2006calibrating}, to perturb the gradients before uploading them, which ensures that any change to the mechanism's input does not significantly affect the output.
The protection level of LDP for owner $i$ is parameterized by $\epsilon_i$, which also quantifies her privacy loss.
A smaller $\epsilon_i$ corresponds to a higher protection level and a more randomized perturbation.
We let $\mathcal{M}_{\epsilon_i}$ denote a perturbation mechanism that satisfies $\epsilon_i$-LDP.
% We note that the following Laplace mechanism also guaranttes $\epsilon_i$-LDP for each owner $i$'s whole dataset.
Note that if we perturb a gradient $g_i$ by $\mathcal{M}_{\epsilon_i}$, releasing the perturbed gradient also satisfies $\epsilon_i$-LDP for each record $r\in D_i$.

\begin{definition}[$\epsilon_i$-Local Differential Privacy \cite{evfimievski2003limiting}]
\label{def:ldp}
Given a privacy loss $\epsilon_i \geq 0$, a randomized mechanism $\mathcal{M}$ satisfies $\epsilon_i$-LDP if for any two inputs $x, x'\in Domain(\mathcal{M})$ and any output $o\in Range(\mathcal{M})$, we have:
\begin{equation*}
\setlength\abovedisplayskip{3pt}%shrink space
\setlength\belowdisplayskip{3pt}
    Pr[\mathcal{M}(x)=o] \leq \exp{(\epsilon_i)} \cdot Pr[\mathcal{M}(x')=o]
\end{equation*}
\end{definition}

% \begin{theorem}[Laplace Mechanism \cite{dwork2014algorithmic} for Gradients]
% \label{thm:lap}
% Given a privacy loss $\epsilon_i \geq 0$ and a dataset $D_i$, the Laplace mechanism $\mathcal{L}_{\epsilon_i}^{L}$ takes as input a gradient $g_i$ derived from Equation \ref{eq:gradient} and returns $\mathcal{L}_{\epsilon_i}^{L}(g_i)=g_i + \mathcal{Z}^{|g_i|}$, where $\mathcal{Z}^{|g_i|}$ is a vector of random variables drawn from the Laplace distribution $Lap(\frac{2L}{\epsilon_i})$ and $|g_i|$ is the dimension of $g_i$.
% The $\mathcal{L}_{\epsilon_i}^L$ mechanism satisfies $\epsilon_i$-LDP for each record $r_{i,j}\in D_i$.
% \end{theorem}

\section{FL-Market Framework}
\label{sec:framework}

% In this section, we propose the FL-Market framework. 
% % We first briefly introduce the participants in FL-Market, then present how we aggregate perturbed local gradients into a perturbed global gradients, and how we trade those gradients at auction. 
% First we provide a brief overview of FL-Market (Section \ref{subsec:overview}). 
% Then, we introduce an auction-based trading framework for FL-Market (Section \ref{subsec:auc}).
% Finally, we formulate our mechanism design problems (Section \ref{sec:problem_formulation}).

\subsection{Market Setup}
\label{subsec:overview}
\subsubsection*{Participants}
As shown in Figure \ref{fig:market_overview}, there are three parties in FL-Market: data owners, model buyers, and an FL broker. 
A \textit{model buyer} enters FL-Market to purchase a global gradient with a financial budget $B$ at each FL training round $r$ to train her target model $h_{w^r}$. 
We assume that the buyer already knows that data owners' data attributes meet her needs.
\textit{Data owners} $\mathcal{N}=\{1,...,n\}$ possess local datasets $D=\{D_1,...,D_n\}$ that can be used to compute local gradients $g_1,...,g_n$ for training $h_{w^r}$. 
%following the FedSGD algorithm (see Section \ref{sec:fl}). 
To prevent privacy leakage against the FL broker and model buyers, each owner $i$ perturbs her local gradient $g_i$ using a perturbation mechanism $\mathcal{M}_{\epsilon_i}$ that satisfies $\epsilon_i$-LDP.
The \textit{broker} mediates between the model buyer and data owners in the FL process: it arranges the training tasks among data owners, collects their perturbed local gradients, and aggregates them into a perturbed global gradient for the buyer. 
In addition, the broker sets the payments $p_1,...,p_n$ to data owners within the buyer's budget $B$. 
% We assume that all the participants follow the protocol but may attempt to infer information from received messages.

\subsubsection*{Privacy valuation}
Inspired by \cite{ghosh2011selling, ghosh2015selling}, FL-Market requires data owners to report their privacy valuations to price perturbed gradients. 
% However, to protect privacy, those local gradients should be perturbed under LDP before submission.
% We employ local differential privacy (LDP) to protect data owners' privacy.
% As mentioned in Section \ref{sec:ldp}, data owners' private information can be inferred from their local gradients $g_i$. 
% To prevent privacy leakage from data owners to the FL broker, we let each data owner $i$ perturb her local gradient $g_i$ by the Laplace mechanism $\mathcal{L}_{\epsilon_i}^{L}$ and decide the paramter $\epsilon_i$. 
% Then, the broker only collects the perturbed local gradients $\tilde{g}_i=\mathcal{L}_{\epsilon_i}(g_i)^{L}$ from data owners to return a perturbed global gradient $\tilde{g}$ to the model buyer.
Concretely, each owner $i$ has a \textit{valuation function} $v_i(\epsilon_i, d_i)$ that reflects her valuation of her privacy loss $\epsilon_i$ for her $d_i$-sized dataset: she will accept a privacy loss $\epsilon_i$ for $d_i$ records if she obtains a payment $p_i \geq v_i(\epsilon_i, d_i)$. 
However, in \cite{ghosh2011selling, ghosh2015selling}, data owners cannot set the upper bounds of their privacy losses.
To provide better privacy protection as an incentive, we follow Zheng et al. \cite{zheng2020money} to allow each owner $i$ to set a \textit{privacy budget} $\bar{\epsilon}_i$ that denotes the maximum tolerable privacy loss. 
% We assume that each valuation function $v_i(\cdot)$ is nondecreasing and $v_i(0)=0$ since when $\epsilon=0$, the gradient is uniformly randomized and does not contain any useful information.
In practice, the broker can provide some instructions to help data owners decide privacy valuations and budgets, e.g., questionnaires for figuring out privacy preferences, typical choices for different preferences, and some analysis of historical transaction data.

\subsubsection*{Threat model}
We assume that all the participants are honest-but-curious, which means they will not deviate from the protocol but will attempt to learn information from received messages.
Note that in an auction, reporting a \textit{fake bid} that does not represent the bidder's real preference is not a malicious behavior that violates the protocol since the auction allows bidders to submit arbitrary bids.

\begin{figure*}[ht]
    \centering
    \includegraphics[scale=0.4]{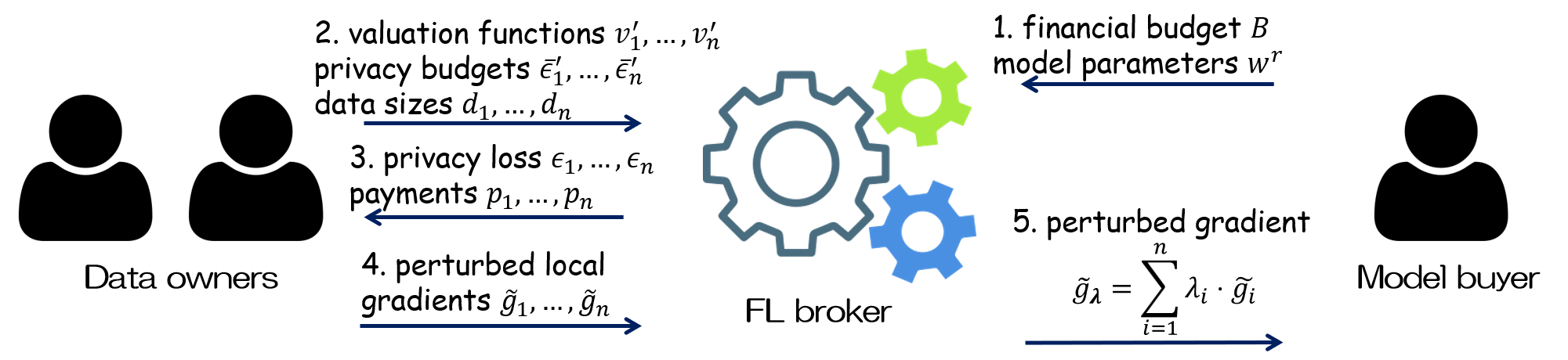}
    \caption{FL-Market Trading Framework.}
    \label{fig:market_overview}
        % \vspace{-6pt}
\end{figure*}

\subsection{Trading Framework}
\label{subsec:auc}

%To facilitate trustworthy ML-oriented data acquisition, we design an auction-based framework for trading locally differentially private gradients. 
We depict the trading framework in Fig. \ref{fig:market_overview} and Alg. \ref{alg:framework}. 
Initially, a model buyer enters FL-Market and specifies a target model $h_w(\cdot)$ with a loss function $l(\cdot)$ for FL.
Then, in each FL training round $r$, the buyer purchases a global gradient for model updating by the following steps:
%FL-Market adopts an auction-based framework  for trading private gradients. 
%One the one hand, data owners can flexibly report their valuations on their private gradients at auction, so that the corresponding payments can motivate their participation in FL. 
%On the other hand, a truthful auction mechanism can prevent untruthful and overvalued valuations from winning the auction, which is also an incentive to truthful data owners.
% Because data owners have full autonomy to participate in or leave from FL, the FL broker employs an auction mechanism to flexibly determine the payments to the them based on their valuations on privacy, so that they are motivated enough to make contributions.

\begin{algorithm}[h]
\small
\caption{Trading Framework of FL-Market}
\begin{algorithmic}[1]
% \REQUIRE target model $h_{w^r}$, loss function $l(\cdot)$, financial budget $B$
% \ENSURE a global perturbed gradient $\tilde{g}_{\boldsymbol{\lambda}}$, payments $p_1,...,p_n$
\STATE A buyer specifies a model $h_{w}(\cdot)$ with a loss function $l(\cdot)$.
\FOR{each FL training round $r$}
\STATE The buyer announces an auction with a financial budget $B$ and model parameters $w^r$.
\STATE Data owners report their bids $\boldsymbol{b}'=(b_1',...,b_n')$.
\STATE The broker runs $\textsf{Auc}(\boldsymbol{b}', B) \to \boldsymbol{\epsilon}, \boldsymbol{p}$.
% $\epsilon_1,...,\epsilon_n, p_1,...,p_n \gets \texttt{\textbf{Auc}}(\boldsymbol{b}', B)$.
\STATE The broker broadcasts $w^{r}$ and data owners compute perturbed local gradients $\tilde{g}_1,...,\tilde{g}_n$.
\STATE  The broker runs $\textsf{Aggr}(\boldsymbol{\epsilon}, \boldsymbol{d}) \to  \boldsymbol{\lambda}$.
%\textit{Aggregation Mechanism} ${\boldsymbol{\lambda}} =  ({\lambda}_1,...,{\lambda}_n) \gets \textsf{\textbf{Aggr}}(\epsilon_1,...,\epsilon_n)$.
\STATE The broker delivers a global gradient $\tilde{g}_{\boldsymbol{\lambda}}=\sum_{i=1}^{n} {\lambda}_i \cdot \tilde{g}_i$ to the buyer for model updating.
\ENDFOR
\end{algorithmic}
\label{alg:framework}
\end{algorithm}
% \vspace{-10pt}

\begin{enumerate}[leftmargin=*]
    \item \textbf{Auction announcement}: The buyer asks the FL broker to announce a procurement auction (where bidders are sellers) for purchasing gradients, specifying a financial budget $B$ and model parameters $w^r$.
    \item \textbf{Bidding}: Data owners report their bids $b_1',...,b_n'$ in the auction.
    We assume that each owner $i$ has a \textit{real bid} $b_i=(v_i, \bar{\epsilon}_i, \bar{d}_i)$ in mind consisting of her valuation function $v_i$, the maximum privacy budget $\bar{\epsilon}_i$, and the maximum size of her dataset $\bar{d}_i$.
    Then, each $i$ reports to the broker a valuation function $v_i'$, a privacy budget $\bar{\epsilon}_i'$ and a data size $d_i$ as a \textit{reported bid} $b_i'=(v_i', \bar{\epsilon}_i', d_i)$.
    If the reported bid $b_i'$ is \textit{truthful}, then $b_i' = b_i$; otherwise, it is a \textit{fake bid}, i.e., $b_i' \neq b_i$.
    % However, strategic data owners may untruthfully report a \textit{fake bid} $b_i'$ to gain a higher utility, which is unfair and an disincentive to those truthful data owners. 
    We simplify "reported bid" as "bid" and denote the collection of all the bids as a \text{bid profile} $\boldsymbol{b}'=[b_1',...,b_n']$.
    % We note that, since the \textit{real bid} $b_i$ is private and only known to data owner $i$, the (reported) bid $b_i'$ might be a \textit{fake bid} which is different from the real bid $b_i$. 
    % We note that reporting a fake bid is not malicious and is allowed by the protocol.
    \item \textbf{Auction decision}: 
    % Privacy loss and payment decision
    The broker runs an \textit{auction mechanism} $\textsf{Auc}$ to decide data owners' privacy losses and payments. 
    Formally, an {auction mechanism} given a bid profile $\boldsymbol{b}'$ and a financial budget $B$ yields an allocation of privacy losses $\boldsymbol{\epsilon}=[\epsilon_1,...,\epsilon_n]$ and payments $\boldsymbol{p}=[p_1,...,p_n]$.
    % We denote the set of winners as $\mathcal{W}$, i.e., those data owners whose privacy losses are positive.
    \item \textbf{Local gradient computing}: 
    Given model parameters $w^r$, each data owner $i$ computes and submits a noisy gradient $\tilde{g}_i=\mathcal{M}_{\epsilon_i}(g_i)$ to the broker. 
    \item \textbf{Gradients aggregation and model delivery}: 
    %Gradient aggregation and delivery
    The FL broker runs an \textit{aggregation mechanism} $\textsf{Aggr}$ to aggregate those noisy gradients into a perturbed global gradient $\tilde{g}_{{\boldsymbol{\lambda}}}$. Finally, the broker returns $\tilde{g}_{{\boldsymbol{\lambda}}}$ to the model buyer.
    % $\textsf{Aggr}$ to decide the aggregation weights ${\boldsymbol{\lambda}} = \textsf{Aggr}(\epsilon_1,...,\epsilon_n)$ and aggregate $\tilde{g}_1,...,\tilde{g}_n$ into a perturbed global gradient $\tilde{g}_{{\boldsymbol{\lambda}}}$ by Equation $\ref{eq:aggr}$. Finally, the broker returns $\tilde{g}_{{\boldsymbol{\lambda}}}$ to the model buyer.
\end{enumerate}

% For simplicity and to avoid distractions, we assume that local datasets $D_1,...,D_n$ are of the same size, i.e., $|D_1|=...=|D_n|$.
% If a data owner has an extraordinarily larger dataset, she can divide it into smaller datasets and participate in FL-Market as multiple times.
% However, in Section \ref{sec:discuss}, we will provide a detailed discussion on how to extend our framework and techniques to facilitate reporting various data sizes for auction.

\subsubsection*{Gradients aggregation}
%In this paper, we focus on a generalized version of the conventional method for gradient aggregation: weighted aggregation \cite{mohri2019agnostic, mcmahan2016federated}. 
In step (5), the broker needs a ``good'' strategy to aggregate the collected noisy gradients.
To study the optimality of the aggregation mechanism in our setting, we generalize the problem as follows.
Formally, given data owners' perturbed gradients $\tilde{g}_1,...,\tilde{g}_n$, the broker sets the \textit{aggregation weights} ${\boldsymbol{\lambda}} = [{\lambda}_1,...,{\lambda}_n]$ with $\sum_{i=1}^n {\lambda}_i = 1, {\lambda}_i \geq 0, \forall i$ and then computes the perturbed global gradient as:
\begin{equation}
\label{eq:aggr}
    \tilde{g}_{\boldsymbol{\lambda}}=\sum_{i=1}^{n} {\lambda}_i \cdot \tilde{g}_i.
\end{equation}
We note that Equation \eqref{eq:aggr} is a generalization of the \textit{weighted aggregation} \cite{mohri2019agnostic, mcmahan2016federated} in the literature.
Then, we attempt to design an optimal aggregation mechanism under personalized privacy losses.
Specifically, we define the aggregation mechanism as a function $\textsf{Aggr}: R^{2n} \!\rightarrow\! R^n$ that given privacy losses $\boldsymbol{\epsilon}\!=\![\epsilon_1,...,\epsilon_n]$ and data sizes $\boldsymbol{d}\!=\![d_1,...,d_n]$ outputs aggregation weights ${\boldsymbol{\lambda}}\!=\![{\lambda}_1,...,{\lambda}_n]$ for weighted aggregation.

% After the transaction, each data owner $i$ obtains a utility
% \begin{equation*}
%     u_i(b_i'; \boldsymbol{b}_{-i}')=
%     \begin{cases}
%     p_i - v_i(\epsilon_i), & \epsilon_i \leq \bar{\epsilon}_i \\
%     -\infty, & \epsilon_i > \bar{\epsilon}_i
%     \end{cases}
% \end{equation*}
% where $\boldsymbol{b}_{-i}'=(b_1',...,b_{i-1}',b_{i+1}',...,b_n')$ denotes other data owners' bids.

\subsection{Mechanism Design} 
%\subsection{Problem Formulation} 
\label{sec:problem_formulation}
In this section, we formulate the problems of designing the auction mechanism $\textsf{{Auc}}$ and aggregation mechanism $\textsf{Aggr}$ (Lines $5$ and $7$ in Alg. \ref{alg:framework}, respectively) to instantiate the trading protocol of FL-Market. 
% To design a trading protocol for FL-Market, i.e., we need to instantiate the modules $(\textsf{{Auc}},\textsf{Aggr})$.
%In this paper, we study how to design a trading protocol by instantiating the modules $(\textsf{{Auc}},\textsf{Aggr})$ in Algorithm \ref{alg:framework} 
The mechanism design should achieve the following two goals: 
(1) to provide \textit{utility-optimal} global gradients and (2) to prevent \textit{untruthful} privacy valuations.

\subsubsection*{\textbf{Aggregation mechanism}}
The aggregation mechanism should optimally aggregate perturbed local gradients to provide highly usable global gradients for model buyers.
Concretely, given local gradients $\tilde{g}_1,...,\tilde{g}_n$ with privacy losses $\boldsymbol{\epsilon}$ and data sizes $\boldsymbol{d}$, $\textsf{Aggr}$ should yield optimal aggregation weights that minimize the error of the global gradient: 
% We fist analyze the utility of aggregation mechanism based on Eq.\eqref{eq:aggr}.
% The error of the perturbed global gradients is:
%On the other hand, the model buyer obtains a perturbed global gradient $\tilde{g}_{{\boldsymbol{\lambda}}}$ of an error
\begin{equation*}
    \min_{{\boldsymbol{\lambda}}} err(\tilde{g}_{{\boldsymbol{\lambda}}}; \boldsymbol{\epsilon}, \boldsymbol{d})=||\tilde{g}_{{\boldsymbol{\lambda}}}-g^*||_2=||\sum_{i=1}^{n}{\lambda}_i\cdot\mathcal{M}_{\epsilon_i}(g_i)-g^*||_2
\end{equation*}
where $g^*=\sum_{i=1}^{n}\frac{d_i}{\sum_{j=1}^{n}d_j}g_i$ is the raw global gradient without any perturbation. The lower the error $err(\tilde{g}_{{\boldsymbol{\lambda}}})$ is, the smaller the difference between $\tilde{g}_{{\boldsymbol{\lambda}}}$ and $g^*$, which also implies that the buyer will obtain a more accurate global model. 

%Our goal is to find a Then, we achieve the second goal by solving a \textit{error bound minimization problem}. Intuitively, the lower difference between the perturbed global gradient $\tilde{g}_{\boldsymbol{\lambda}}$ and the raw one $g^*$ can result in a more accurate FL model under LDP. 
% Therefore, we want to minimize the error of $\tilde{g}_{\boldsymbol{\lambda}}$ by allocating the optimal aggregation weights ${\boldsymbol{\lambda}}$ and privacy losses $\epsilon_1,...,\epsilon_n$:
% \begin{equation*}
% \label{eq:err}
% \setlength\abovedisplayskip{3pt}%shrink space
% \setlength\belowdisplayskip{3pt}
%     \min_{{\boldsymbol{\lambda}}, \epsilon_1,...,\epsilon_n} err(\tilde{g}_{{\boldsymbol{\lambda}}}) \quad (P1)
% \end{equation*}
However, the broker cannot calculate the ground-truth error $err(\tilde{g}_{{\boldsymbol{\lambda}}})$ under LDP without the access to $g_1,..,g_n$. Hence, we turn to the error bound $ERR(\tilde{g}_{{\boldsymbol{\lambda}}})$ and design the aggregation mechanism by solving the following problem:

\begin{problem}[Error Bound-Minimizing Aggregation]
\label{problem:error_mini}
\begin{align*}
% \vspace{-6pt}
% \setlength\abovedisplayskip{5pt}%shrink space
% \setlength\belowdisplayskip{5pt}
         \min_{{\boldsymbol{\lambda}}=\textsf{Aggr}(\boldsymbol{\epsilon},\boldsymbol{d})} & ERR(\tilde{g}_{{\boldsymbol{\lambda}}}; \boldsymbol{\epsilon}, \boldsymbol{d}) =\sup_{g_1,...,g_n}err(\tilde{g}_{{\boldsymbol{\lambda}}}; \boldsymbol{\epsilon}, \boldsymbol{d})\\
        \text{S.t.: }  & \forall i, {\lambda}_i \in [0,1], \text{ and }  \sum_{i=1}^{n} {\lambda}_i=1
% \vspace{-6pt}
\end{align*}
\end{problem}

% \begin{align*}
% \vspace{-6pt}
% \setlength\abovedisplayskip{5pt}%shrink space
% \setlength\belowdisplayskip{5pt}
%         &\min_{{\boldsymbol{\lambda}}, \epsilon_1,...,\epsilon_n} \sup_{g_1,...,g_n} err(\tilde{g}_{{\boldsymbol{\lambda}}}) \quad (P1.1)\\
%         \triangleq &\min_{{\boldsymbol{\lambda}}, \epsilon_1,...,\epsilon_n} ERR(\tilde{g}_{{\boldsymbol{\lambda}}}) \quad \text{(when $|D_1|=...=|D_n|$)} \\
%         =&\sup_{g_1,...,g_n}||\sum_{i=1}^{n}{\lambda}_i\cdot\mathcal{L}_{\epsilon_i}^{L}(g_i)-\sum_{i=1}^{n}\frac{1}{n}g_i||_2 \quad (P1.2)\\
%         \text{S.t.: }&\sum_{i=1}^{n} p_i(\epsilon_i) \leq B \text{ (Budget Constraint)}\\
%         & \forall i, \epsilon_i \in [0, \bar{\epsilon}_i'], {\lambda}_i \in [0,1], \text{ and }  \sum_{i=1}^{n} {\lambda}_i=1
% \vspace{-6pt}
% \end{align*}

\subsubsection*{\textbf{Auction mechanism}}
% For the first goal, we should provide enough incentives with data owners to participant in FL-Market. 
% Concretely, we guarantee some important economic incentives on the payments to them in our auction mechanism.
Solving Problem \ref{problem:error_mini} alone is still insufficient to determine a utility-optimal global gradient since the utility is also affected by the privacy losses purchased for perturbing the local gradients.
% An auction mechanism $\textsf{Auc}$ should realize both of the above design goals, since it determines privacy losses that can affect the utility of gradients and the payments to data owners.
That is, the auction mechanism $\textsf{Auc}$ should take the aggregation mechanism into account to \textit{jointly} optimize the (expected) error bound of the global gradient over all possible bid profiles and financial budgets:
\begin{align*}
    &\min_{\boldsymbol{\epsilon}, \boldsymbol{p} 
    =\textsf{Auc}(\boldsymbol{b}', B)} \mathbb{E}_{(\boldsymbol{b}', B)} [ERR(\tilde{g}_{{\boldsymbol{\lambda}}}; {\boldsymbol{\lambda}}=\textsf{Aggr}(\boldsymbol{\epsilon}, \boldsymbol{d}))] \\
    =&\mathbb{E}_{(\boldsymbol{b}', B)} [\sup_{g_1,...,g_n} ||\sum_{i=1}^{n}{\lambda}_i\cdot\mathcal{M}_{\epsilon_i}(g_i)-g^*||_2]\\
    =& \mathbb{E}_{(\boldsymbol{b}', B)} [\sup_{g_1,...,g_n} ||\textsf{Aggr}(\boldsymbol{\epsilon}, \boldsymbol{d})\cdot[\mathcal{M}_{\epsilon_1}(g_1),...,\mathcal{M}_{\epsilon_n}(g_n)]-g^*||_2]
\end{align*}

% Formally, $\textsf{Auc}$ should allocate optimal privacy losses that will result in a perturbed global gradient with the expected minimized error (under a fixed aggregation mechanism $\textsf{Aggr}(\cdot)$) over all possible bid profiles and financial budgets:
% \begin{align*}
%     \min_{\boldsymbol{\epsilon}, \boldsymbol{p} =\textsf{Auc}(\boldsymbol{b}', B)} & \mathbb{E}_{(\boldsymbol{b}', B)} err(\tilde{g}_{{\boldsymbol{\lambda}}}; \textsf{Aggr}(\cdot))=\mathbb{E}_{(\boldsymbol{b}', B)} ||\sum_{i=1}^{n}{\lambda}_i\cdot\mathcal{L}_{\epsilon_i}^{L}(g_i)-g^*||_2\\
%     =& \mathbb{E}_{(\boldsymbol{b}', B)} ||\textsf{Aggr}(\epsilon_1,...,\epsilon_n)\cdot[\mathcal{L}_{\epsilon_1}^{L}(g_1),...,\mathcal{L}_{\epsilon_n}^{L}(g_n)]-g^*||_2
% \end{align*}

Then, $\textsf{Auc}$ needs to determine appropriate auction results that prevent untruthful privacy valuations.
Concretely, by trading a global gradient, each data owner $i$ obtains a utility 
\begin{equation}
    u_i(b_i'; \boldsymbol{b}_{-i}', B)=
    \begin{cases}
    p_i - v_i(\epsilon_i, d_i), & \epsilon_i \leq \bar{\epsilon}_i, d_i\leq \bar{d}_i\\
    -\infty, & \text{otherwise}
    \end{cases}
    \nonumber
\end{equation} 
where $\boldsymbol{b}_{-i}'=(b_1',...,b_{i-1}',b_{i+1}',...,b_n')$ denotes the other bidders' bids. Then, $\textsf{Auc}$ should ensure the following incentives:
\begin{itemize}[leftmargin=*]
    \item Truthfulness: With the other bidders' bids $\boldsymbol{b}_{-i}'$ fixed, each bidder $i$ never obtains a higher utility by reporting a fake bid $b_i'\neq b_i$, i.e., $\forall i, \forall b_i', \forall B, u_i(b_i';\boldsymbol{b}_{-i}', B) \leq u_i(b_i; \boldsymbol{b}_{-i}', B)$.
    % \begin{equation*}
    % \setlength\abovedisplayskip{3pt}%shrink space
    % \setlength\belowdisplayskip{3pt}
    %     \forall i, \forall b_i', u_i(b_i';\boldsymbol{b}_{-i}') \leq u_i(b_i; \boldsymbol{b}_{-i}')
    % \end{equation*}
    % \item Budget Feasibility (BF): The sum of the payments $p_1,...,p_n$ to data owners should be covered by the model buyer's financial budget $B$, i.e., $\sum_{i=1}^{n} p_i \leq B$.
    \item Individual rationality (IR): Each bidder $i$ never obtains a negative utility, i.e., $u_i(b_i')\geq 0, \forall b_i', \forall i$.
    % \item Computational efficiency (CE): The auction mechanism $\textsf{Auc}$ can determine payments in polynomial time.
    \item Budget feasibility (BF): The payments should be within the financial budget, i.e., $\sum_i p_i \leq B$.
\end{itemize}
Therefore, we can design the auction mechanism by solving the following problem.

\begin{problem}[Budget-Limited Multi-Unit Multi-Item Procurement Auction]
\label{problem:auction}
\begin{align*}
% \vspace{-6pt}
% \setlength\abovedisplayskip{5pt}%shrink space
% \setlength\belowdisplayskip{5pt}
         &\min_{\boldsymbol{\epsilon}, \boldsymbol{p}=\textsf{Auc}(\boldsymbol{b}', B)}  \mathbb{E}_{(\boldsymbol{b}', B)} [ERR(\tilde{g}_{{\boldsymbol{\lambda}}}; {\boldsymbol{\lambda}}=\textsf{Aggr}(\boldsymbol{\epsilon}, \boldsymbol{d}))]\\
        \text{S.t.: }  & \forall i, \epsilon_i \in [0, \bar{\epsilon}'_i],  \text{ truthfulness, IR, and BF.}
% \vspace{-6pt}
\end{align*}
\end{problem}
% where $ERR(\tilde{g}_{{\boldsymbol{\lambda}}}; \textsf{Aggr}(\cdot)) = \sup_{g_1,...,g_n} err(\tilde{g}_{{\boldsymbol{\lambda}}}; \textsf{Aggr}(\cdot))$.
Problem \ref{problem:auction} is a budget-limited multi-unit multi-item procurement auction problem \cite{chan2014truthful} because (1) each data owner's privacy loss $\epsilon_i$ can be seen as a divisible item for procurement with $\bar{\epsilon}'_i$ units available, and (2) the buyer purchases privacy losses under her financial budget $B$.
To the best of our knowledge, such a problem has yet to be generally solved in the literature.
Moreover, we have to involve the aggregation mechanism in minimizing the global gradient's error bound, which increases the complexity of optimal mechanism design.
Concretely, the privacy losses affect the aggregation weights in Problem \ref{problem:error_mini}, but the latter also feeds back into the former in Problem \ref{problem:auction}, which calls for joint optimization.
By solving this problem, we can obtain an auction mechanism that maximizes the global gradient's utility jointly with \textsf{Aggr}.
% In this paper, we will propose two approximate solutions by (1) making an assumption on data owners' valuation functions and (2) slightly relaxing the truthfulness constraint.

\subsubsection*{\textbf{Computational efficiency}}
We additionally require that the auction and aggregation mechanisms (designed by solving Problems \ref{problem:error_mini} and \ref{problem:auction}) should finish in polynomial time, which ensures the efficiency of FL-Market.
Note that we design the mechanisms offline before executing Algorithm \ref{alg:framework} rather than during each FL training round therein.

%\subsubsection*{Solutions}
%In Section \ref{sec:ag}, we will first propose aggregation mechanisms to solve the error bound minimization problem by optimally aggregating the perturbed local gradients . Then, in Section \ref{sec:ac}, we will propose auction mechanisms that not only minimizes the error bound by properly selecting data owners' privacy losses, but also decides appropriate payments with the guarantees of those economic incentives.

% \begin{theorem}
% \label{thm:nph}
% Both P2.1 and P2.2 are NP-hard.
% \end{theorem}

% \begin{proof}
% All the missing proofs can be found in Appendix \ref{sec:appendix}. 
% \end{proof}

% \vspace{-6pt}
\section{Aggregation Mechanism: OptAggr}
\label{sec:ag}
% In this section, first we analyze the error bound of the perturbed global gradient, and then propose two gradient aggregation mechanisms as instances of the module $\textsf{Aggr}$, i.e., \textit{BiasOpt} and \textit{VarOpt}, which minimize the bias bound and variance bound, respectively. 
% In this section, first we adapt the conventional gradient aggregation method (i.e., data size-weighted averaging \cite{mcmahan2016federated}) into our personalized LDP setting and propose our baseline solution, which can minimize the bias bound of the perturbed global gradient. 
% However, since the Laplace mechanism is unbiased but introduces high variance into the global gradients due to personalized privacy losses, we are badly in need of a variance-optimal solution. Hence, we further propose an advanced solution that can minimize the aggregated gradients' variance.
In this section, we propose an error-optimal aggregation mechanism \textit{OptAggr} by solving a \textit{convex quadratic programming problem} that we prove is equivalent to Problem \ref{problem:error_mini}.

\subsubsection*{Error bound decomposition}
It is well known that the MSE error of a random variable consists of its variance and squared bias.
Let $\sigma_i$ denote the variance of the local gradient $\tilde{g}_i$, and let $W_i = \frac{d_i}{\sum_{j\in [n]} d_j}, \forall i$.
We can decompose the error $err(\tilde{g}_{{\boldsymbol{\lambda}}}; \boldsymbol{\epsilon},\boldsymbol{d})$ as $err(\tilde{g}_{{\boldsymbol{\lambda}}}; \boldsymbol{\epsilon},\boldsymbol{d})=var(\tilde{g}_{{\boldsymbol{\lambda}}}; \boldsymbol{\epsilon}) + {bias}^2(\tilde{g}_{{\boldsymbol{\lambda}}}; \boldsymbol{\epsilon},\boldsymbol{d})$ where

\begin{align*}
    & var(\tilde{g}_{{\boldsymbol{\lambda}}}; \boldsymbol{\epsilon})=var(\sum_{i=1}^{n}{\lambda}_i \tilde{g}_i;\boldsymbol{\epsilon} ) = \sum_{i=1}^{n}({\lambda}_i)^2 \sigma_i,\\
    & {bias}(\tilde{g}_{{\boldsymbol{\lambda}}}; \boldsymbol{\epsilon},\boldsymbol{d})=||E[\tilde{g}_{\boldsymbol{\lambda}}] - E[g^*] ||_2 = \medmath{||\sum_{i=1}^{n} {\lambda}_i g_i -\sum_{i=1}^{n} W_i g_i ||_2}\\
    = & || \sum_{i=1}^{n} ({\lambda}_i - W_i) g_i ||_2 \leq \sum_{i=1}^{n} |{\lambda}_i - W_i| \cdot ||g_i||_2 \\
    = & \medmath{\sum_{i=1}^{n} |{\lambda}_i - W_i| \cdot ||\mathbb{E}_{r \in D_i}[\nabla l(w^r; r)]\cdot \min(1, \frac{L}{||\mathbb{E}_{r \in D_i}[\nabla l(w^r; r)]||_1})||_2}
\end{align*}
Because $\sup_{g_1,...,g_n} {bias}(\tilde{g}_{{\boldsymbol{\lambda}}}; \boldsymbol{\epsilon},\boldsymbol{d}) = \sum_{i=1}^{n}|{\lambda}_i - W_i|L$, the objective function of Problem \ref{problem:error_mini} is equal to
\begin{equation*}
             \min_{{\boldsymbol{\lambda}}=\textsf{Aggr}(\boldsymbol{\epsilon},\boldsymbol{d})}  ERR(\tilde{g}_{{\boldsymbol{\lambda}}}; \boldsymbol{\epsilon},\boldsymbol{d})  =  \sum_{i=1}^{n}({\lambda}_i)^2 \sigma_i + (\sum_{i=1}^{n}|{\lambda}_i - W_i|L)^2.
\end{equation*}

\subsubsection*{Problem transformation}
We further transform Problem \ref{problem:error_mini} into a convex quadratic programming problem. 
First, to minimize the error bound, any data owner $i$ with $\epsilon_i=0$ must be allocated a zero-valued weight ${\lambda}_i=0$ by an optimal solver because its gradient $\tilde{g}_i$ has an infinite variance $\sigma_i$.
For simplicity, we assume that only the first $k\leq n$ data owners have positive privacy losses without loss of generality.
Then, we let $\boldsymbol{x}=[{\lambda}_1,...,{\lambda}_{k}]$
and replace the terms $|{\lambda}_i-W_i|, \forall i \in [k]$ with auxiliary variables $\boldsymbol{y} = [y_1,...,y_k]$ with the constraints $y_i \geq -({\lambda}_i-W_i), y_i \geq  {\lambda}_i-W_i, \forall i \in [k]$.
Consequently, we have the following quadratic programming problem \cite{stellato2020osqp}.
\begin{problem}[Equivalent problem of Problem \ref{problem:error_mini}]
\label{problem:psd_quad}
\begin{align*}
\small
    & \min_{\boldsymbol{x}, \boldsymbol{y}} \frac{1}{2} 
    \begin{bmatrix}
    \boldsymbol{x} \\ \boldsymbol{y}
    \end{bmatrix}^T 
     \begin{bmatrix}
    Diag([\sigma_1,...,\sigma_k]) & 0\\ 0 & Uni(L^2)_{k\times k}
    \end{bmatrix}
    \begin{bmatrix}
    \boldsymbol{x} \\ \boldsymbol{y}
    \end{bmatrix} \\
    & \text{S.t.: }   
     \begin{bmatrix}
    Uni(1)_{k\times 1} \\ Uni(0)_{k\times 1}
    \end{bmatrix}^T 
    \begin{bmatrix}
    \boldsymbol{x} \\ \boldsymbol{y}
    \end{bmatrix} = 1,
    \begin{bmatrix}
    \boldsymbol{I}_{k} & -\boldsymbol{I}_{k} \\ -\boldsymbol{I}_{k} & -\boldsymbol{I}_{k}
    \end{bmatrix}
    \begin{bmatrix}
    \boldsymbol{x} \\ \boldsymbol{y}
    \end{bmatrix}
    \leq 
     \begin{bmatrix}
   \boldsymbol{W} \\ -\boldsymbol{W}
    \end{bmatrix}
\end{align*}
where $\boldsymbol{I}_{k}$ is a $k\times k$ identity matrix,
$Uni(a)_{m\times n}$ is an $m\times n$ matrix where all the elements are equal to $a\in R$, 
$Diag([\sigma_1,...,\sigma_k])$ is a $k \times k$ diagonal matrix with $Diag([\sigma_1,...,\sigma_k])[i][i]\!=\!\sigma_i, \forall i \in [k]$,
and $\boldsymbol{W}\!= \![{W}_1,...,{W}_k]$.
\end{problem}

\begin{algorithm}[t]
\small
\caption{Aggregation Mech.: OptAggr}
\begin{algorithmic}[1]
\REQUIRE privacy losses $\epsilon_1,...,\epsilon_n$, data sizes $d_1,...,d_n$
\ENSURE aggregation weights ${\lambda}_1,...,{\lambda}_n$
\RETURN ${\lambda}_i=W_i, \forall i$ if $\epsilon_i=0, \forall i$
\STATE For each data owner $i$ with $\epsilon_i=0$, let ${\lambda}_i=0$
\STATE For each data owner $j$ with $\epsilon_j>0$, calculate the variance $\sigma_j$; then compute ${\lambda}_j, \forall j$ using an optimizer that solves Problem \ref{problem:psd_quad}.
\RETURN ${\lambda}_1,...,{\lambda}_n$
\end{algorithmic}
\label{alg:optaggr}
\end{algorithm}
\vspace{-6pt}

Because Problem \ref{problem:psd_quad} is a convex quadratic programming problem, it can be well solved by many existing solvers in polynomial time, e.g., the SCS solver \cite{scs} to be used in our experiments.
Note that there is no existing analytical solution to Problem \ref{problem:psd_quad} to the best of our knowledge.
Hence, we propose the OptAggr mechanism that (1) allocates zero-valued aggregation weights to those data owners with zero-valued privacy losses and (2) then computes other data owners' aggregation weights by solving Problem \ref{problem:psd_quad} with a polynomial-time optimizer, as depicted in Algorithm \ref{alg:optaggr}.

\begin{proposition}
Problem \ref{problem:psd_quad} is a convex quadratic programming problem and is equivalent to Problem \ref{problem:error_mini}.
\end{proposition}

% The missing proof can be found in our extended report \cite{zheng2021fl}.
\begin{proof}
Let $\mathcal{Q}=\medmath{
\begin{bmatrix}
    Diag([\sigma_1,...,\sigma_k]) & 0\\ 0 & Uni(L^2)_{k\times k}
\end{bmatrix}}$ 
and $A=\medmath{
\begin{bmatrix}
    Diag([\sqrt{\sigma_1},...,\sqrt{\sigma_k}]) & 0\\ 0 & Uni(\frac{L}{\sqrt{k}})_{k\times k}
\end{bmatrix}}$.
Because $\mathcal{Q}=A^{T} A$, $\mathcal{Q}$ is a positive semidefinite matrix. Therefore, Problem \ref{problem:psd_quad} is a convex quadratic programming problem.

For each $y_i$, a solver for Problem \ref{problem:psd_quad} will find the lowest value of $y_i$ as possible. Therefore, if ${\lambda}_i - W_i \geq 0$, the constraint $y_i\geq {\lambda}_i - W_i$ is equivalent to $y_i = {\lambda}_i - W_i$ and implies $y_i\geq -({\lambda}_i - W_i)$; if ${\lambda}_i - W_i \leq 0$, the constraint $y_i\geq -({\lambda}_i - W_i)$ is equivalent to $y_i=-({\lambda}_i - W_i)$ and implies $y_i\geq {\lambda}_i - W_i$. Therefore, the constraints $y_i\geq {\lambda}_i - W_i$ and $y_i\geq -({\lambda}_i - W_i)$ are equivalent to $y_i = |{\lambda}_i - W_i|$. Therefore, we conclude that Problem \ref{problem:psd_quad} is equivalent to Problem \ref{problem:error_mini}.
\end{proof}

\section{Auction Mechanism: DM-RegretNet}
\label{sec:ac}

In this section, we design a truthful mechanism that maximizes the global gradient's utility jointly with the OptAggr mechanism.
Since OptAggr does not provide an analytical solution to Problem \ref{problem:psd_quad}, the objective function also cannot be expressed in an analytic form, which makes it extremely difficult to characterize and design an optimal truthful mechanism.
To design a truthful mechanism that optimizes the nonanalytical objective, we turn to an automated mechanism design approach that achieves an auction objective by ML.
We also propose a traditional auction mechanism in \venueforappendix{\ref{sec:all-in}}.

\begin{figure*}[t]
\vspace{-12pt}
\begin{minipage}[t]{0.6\linewidth}
    \centering
    \includegraphics[scale=0.3]{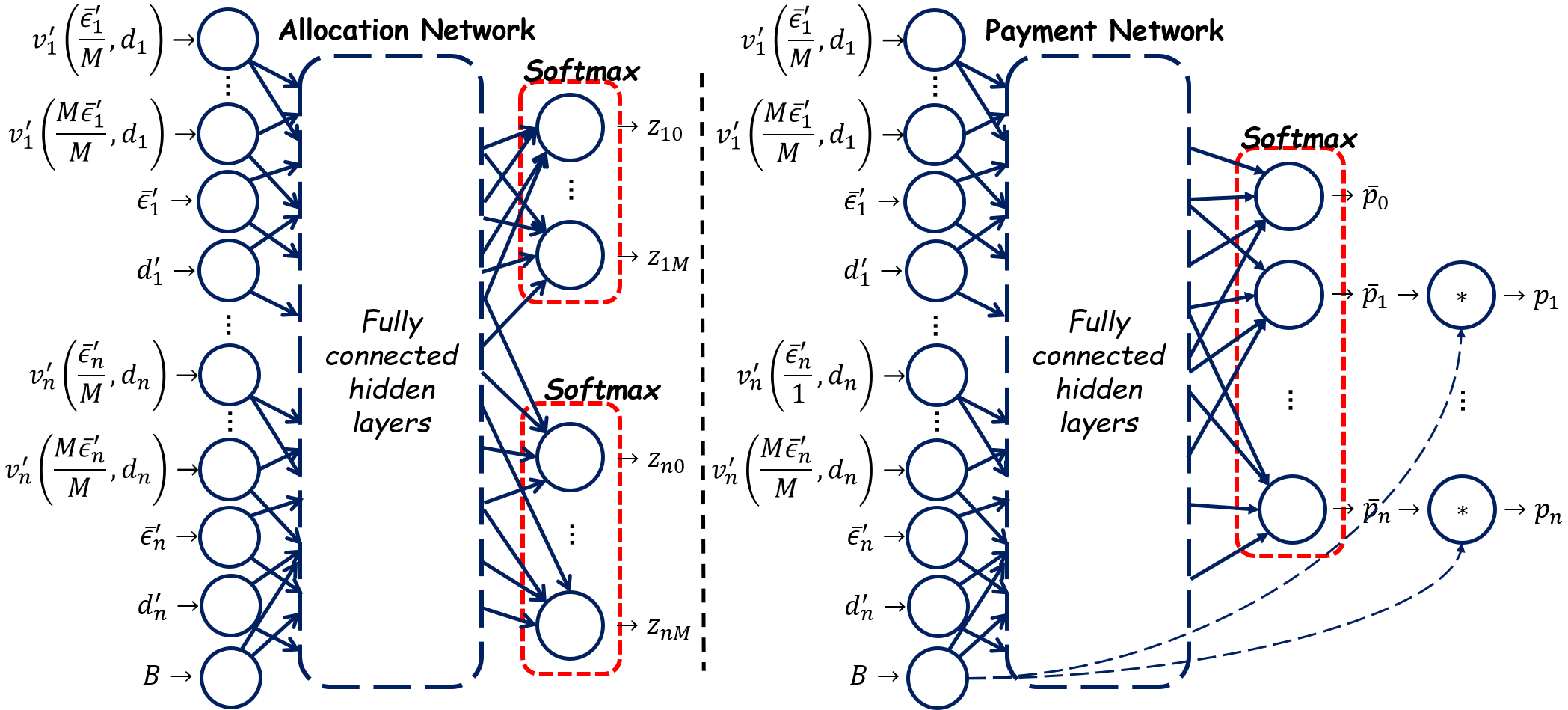}
    % \vspace{-6pt}
    \caption{M-RegretNet.}
    \label{fig:multi-regretnet}
    \end{minipage}%
    \begin{minipage}[t]{0.4\linewidth}
    \centering
    \includegraphics[scale=0.3]{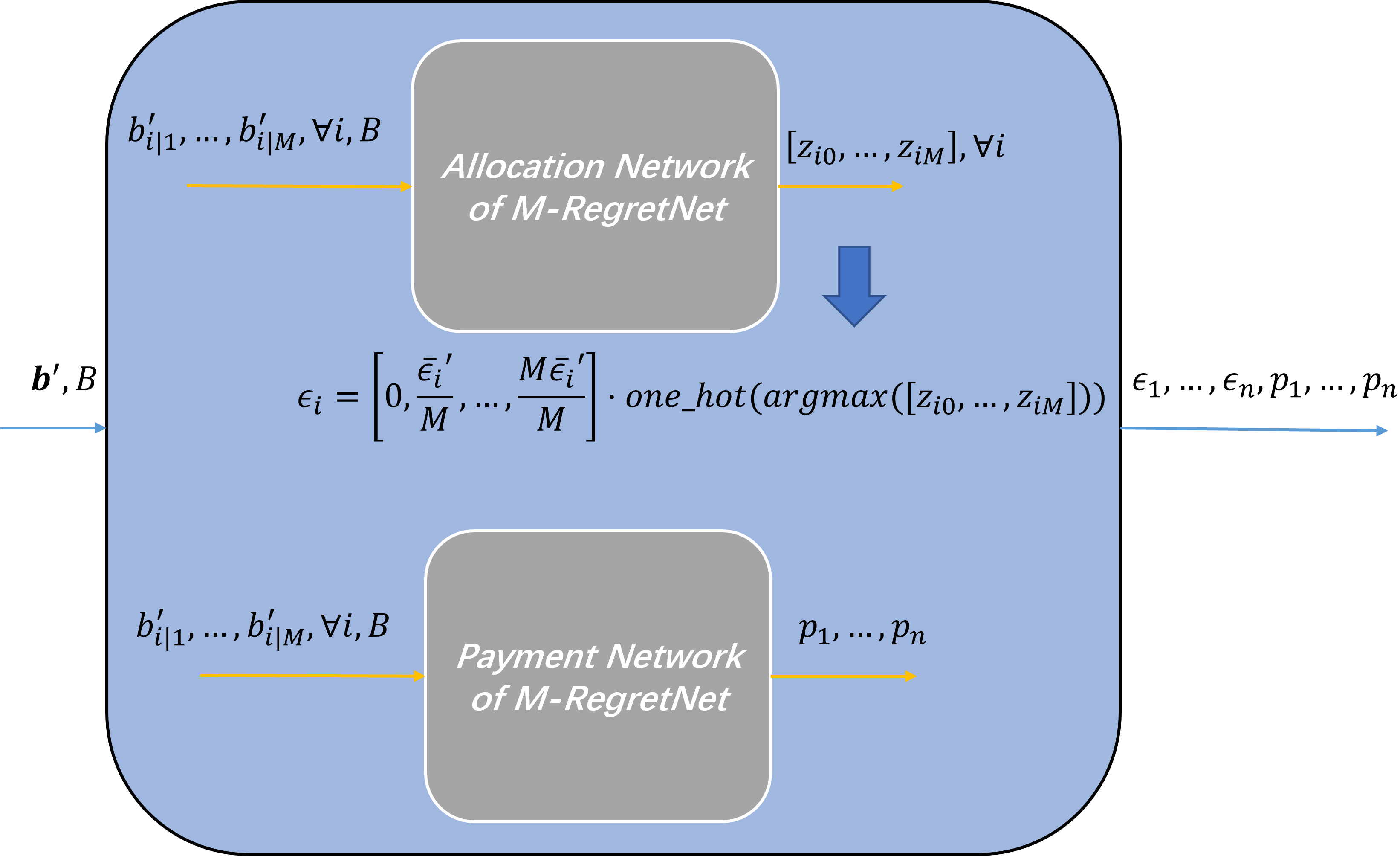}
        % \vspace{-16pt}
    \caption{DM-RegretNet.}
    \label{fig:dm-regretnet}
    \end{minipage}
    \vspace{-12pt}
\end{figure*}

% \subsection{DM-RegretNet: Deterministic M-RegretNet}
% \label{sec:dm}
% Since data owners can only set either all their privacy budgets as privacy losses or zero, there exists much space to minimize the error bound of the perturbed global gradient. 
% In this section, we further drop out the single-minded assumption and try to propose a deep-learning based auction mechanism that supports general valuation functions.
% Considering the limitations of All-in, we further utilize deep learning to design a truthful and error-optimal auction mechanism. 
% Unlike All-in only considering single-minded data owners, we propose a general multi-unit auction mechanism based on the RegretNet framework, i.e., MURBA. 
% It can outperform All-in in terms of minimizing the error bound, because it (1) directly takes minimizing the error bound into its objective function of model training, (2) can select privacy losses more flexibly in a multi-minded way.
\begin{figure}[ht]
    \centering
    \includegraphics[scale=0.45]{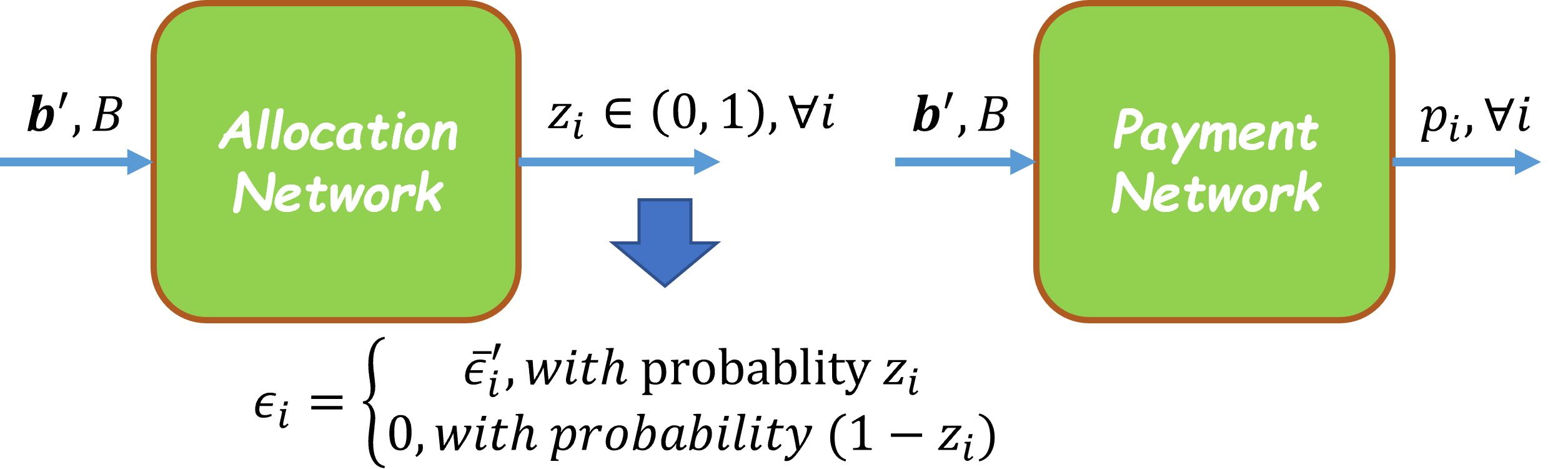}
    \caption{RegretNet.}
    \label{fig:regretnet}
            \vspace{-10pt}
\end{figure}

\subsubsection*{RegretNet}
We seek support from RegretNet \cite{dutting2019optimal}, the state-of-the-art automated mechanism design framework for multi-item auctions. 
As depicted in Figure \ref{fig:regretnet}, RegretNet consists of two deep learning networks: an allocation network and a payment network. 
Both the networks take as input data owners' bid profile $\boldsymbol{b}'$ and the buyer's financial budget $B$ but output allocation probabilities $z_i\in (0,1), \forall i$ and payments $p_1,...,p_n$, respectively.
Therefore, RegretNet is a randomized auction mechanism: the allocation result of each data owner's privacy loss is a binary random variable $\epsilon_i$ with $Pr[\epsilon_i=\bar{\epsilon}'_i]=z_i$ and $Pr[\epsilon_i=0]=1-z_i$.\footnote{The concrete privacy loss $\epsilon_i$ to be used to perturb the local gradient is a sample of the random variable. To reduce notational overload, we use $\epsilon_i$ to denote the random variable in Section \ref{sec:ac}.}
Then, the truthfulness constraint is approximately guaranteed by model training: the violation degree of truthfulness is quantified as a regret penalty in the training objective function to be minimized.

\subsubsection*{Problems with RegretNet}
RegretNet may perform poorly in our auction problem.
First, RegretNet can only auction single-unit items and output binary auction results.
That is, under RegretNet, the allocation result of each data owner $i$'s privacy loss (i.e., the item $\epsilon_i$) is either to purchase the whole unit (i.e., $\epsilon_i=\bar{\epsilon}'_i$) or not to purchase any privacy loss (i.e., $\epsilon_i = 0$).
However, we should support trading a portion of the privacy budget $\bar{\epsilon}'_i$ to flexibly optimize the global gradient's utility.
Second, some extra variance from the randomness of the allocation results by RegretNet might be introduced into the perturbed local gradients.
Third, also because of the allocation randomness, RegretNet cannot treat the (expected) error bound minimization function as the objective function for model training.
Intuitively, RegretNet always allocates zero-valued privacy losses for all data owners with probability $Pr[\epsilon_1=...=\epsilon_n=0]=\prod_{i=1}^n(1-z_i)$, which means that the expected error bound of the perturbed global gradient $\tilde{g}_{{\boldsymbol{\lambda}}}$ remains infinite and cannot be minimized.

\subsubsection*{M-RegretNet}
To solve the first problem with RegretNet, we extend the allocation network of RegretNet and propose \textit{M-RegretNet} (Multi-Unit RegretNet). 
As shown in Figure \ref{fig:multi-regretnet}, like RegretNet, M-RegretNet has an allocation (payment) network with multiple fully connected hidden layers, each of which has multiple hidden nodes with \textit{tanh} activations.
However, it does not take the reported bids as input.
Instead, for each data owner $i$, we transform her reported bid $b_i'=(v_i', \bar{\epsilon}_i', d_i)$ into $M$ sub-bids $b_{i|1}',...,b_{i|M}'$ and then input them into M-RegretNet, where $b_{i|m}'=(v_i'(\frac{m\cdot\bar{\epsilon}_i'}{M}, d_i), \frac{m\cdot \bar{\epsilon}_i'}{M}, d_i), \forall m\in[M]$.  
Regarding the payment network, it first generates \textit{budget fractions} $\bar{p}_0,...,\bar{p}_n$ and then output payments $p_i = \bar{p}_i \cdot B, \forall i \in [n]$.
Because the budget fractions are output by a \textit{softmax} activation function, the sum of the payments $\sum_{i\in [n]} p_i$ never exceeds the financial budget $B$, which ensures BF.
Then, the allocation network outputs $\boldsymbol{z}_i = [z_{i0}, z_{i1},...,z_{iM}]$ for each data owner $i$, where $z_{i0}$ denotes the probability of data owner $i$ losing the auction, and $z_{im}$ is the probability of data owner $i$ winning with her sub-bid $b_{i|m}'$. 
Since each owner $i$ should win with at most one sub-bid, we apply softmax activation functions to ensure that $\sum_{m=0}^{M} z_{im} = 1, \forall i$. 
Therefore, the allocation result for each data owner $i$ is a random variable $\epsilon_i$ with $M+1$ possible values, i.e., $Pr[\epsilon_i = \frac{m\cdot\bar{ \epsilon}_i'}{M}] = z_{im}, \forall m \in \{0, 1, ..., M\}$.
When $M=1$, M-RegretNet reduces to a budget-feasible version of RegretNet;
when $M \geq 2$, it enables the buyer to only purchase a part of each data owner's privacy budget.
In addition, when $M$ increases, it becomes easier for M-RegretNet to approximate the truthfulness and IR guarantees since it has more possible values to allocate as privacy losses.

\subsubsection*{DM-RegretNet}
To address the second and third problems with RegretNet, we further propose \textit{DM-RegretNet} that outputs deterministic allocation results.
% based on a simple but effective insight:
% If an allocation network always outputs a vector $[z_{i0}, z_{i1},...,z_{iM}]$ with a one-valued allocation probability for each $i$, 
% both of the problems can be avoided because the allocation results $\epsilon_1,...,\epsilon_n$ are no longer random variables, i.e., they are deterministic.
DM-RegretNet deploys M-RegretNet as a module to determine allocation probabilities $\boldsymbol{z}_1,...,\boldsymbol{z}_n$ and payments $p_1,...,p_n$.
Then, it realizes deterministic allocation results by processing the vector of allocation probabilities $\boldsymbol{z}_i = [z_{i0},...,z_{iM}]$ into a one-hot vector; by such a process, there is only one one-valued allocation probability for each data owner $i$, and thus each privacy loss $\epsilon_i$ is deterministic.
Formally, it is
\begin{equation}
\label{eq:one-hot}
    \epsilon_i = \medmath{[0, \frac{1 \cdot \bar{\epsilon}_i'}{M},...,\frac{M \cdot \bar{\epsilon}_i'}{M}] \cdot one\_hot({argmax}(\boldsymbol{z}_i))}
\end{equation}
where $one\_hot(\cdot)$ is a function that takes as input an integer $m \in [0, M]$ and outputs an $(M+1)$-length one-hot vector where the $m$-th element equals $1$ and the others are zero-valued.
However, the function $one\_hot(argmax(\cdot))$ is nondifferentiable, which makes the networks untrainable.

\begin{algorithm}[h]
\small
\caption{Auction Mech.: DM-RegretNet}
\label{alg:dm}
\begin{algorithmic}[1]
\REQUIRE (reported) bid profile $\boldsymbol{b}'=(b_1',...,b_n')$, financial budget $B$, the number of sub-bids $M$, training=False
\ENSURE privacy losses $\epsilon_1,..,\epsilon_n$, payments $p_1,...,p_n$
\STATE Transform each data owner $i$ bid $b_i'=(v_i', \bar{\epsilon}_i',d_i)$ into $M$ sub-bids $b_{i|1}',...,b_{i|M}'$, where $b_{i|m}'=(v_i'(\frac{m\cdot \bar{\epsilon}_i'}{M}, d_i),\frac{m\cdot \bar{\epsilon}_i'}{M},d_i), \forall m \in [M]$
\STATE Input sub-bids, privacy budgets, and financial budget into M-RegretNet to obtain $\boldsymbol{z}_1,...,\boldsymbol{z}_n, p_1,...,p_n$ where $\boldsymbol{z}_i = [z_{i0},...,z_{iM}]$
% \STATE Auction decision: $[z_{i0},...,z_{iM}], p_i, \forall i \gets \textsf{M-RegretNet}(v_1'(\frac{\bar{\epsilon}_1'}{1}),...,v_1'(\frac{\bar{\epsilon}_1'}{M}),\bar{\epsilon}_1',...,v_n'(\frac{\bar{\epsilon}_n'}{1}),...,v_n'(\frac{\bar{\epsilon}_n'}{M}),\bar{\epsilon}_n', B)$
% \STATE Privacy loss calculation: $\forall i, \epsilon_i \gets \sum_{m=1}^{M} z_{im}\cdot\frac{\bar{\epsilon}_i'}{m}$
\IF{training $==$ True}
\STATE  $\hat{\epsilon}_i = [0, \frac{1 \cdot \bar{\epsilon}_i'}{M},...,\frac{M \cdot \bar{\epsilon}_i'}{M}] \cdot {softmax}(\frac{\boldsymbol{z}_i}{\tau}), \forall i$
\RETURN $\hat{\epsilon}_1,...,\hat{\epsilon}_n, p_1,...,p_n$
\ELSE 
\STATE $\epsilon_i = [0, \frac{1 \cdot \bar{\epsilon}_i'}{M},...,\frac{M \cdot \bar{\epsilon}_i'}{M}] \cdot one\_hot(argmax(\boldsymbol{z}_i)), \forall i$
\RETURN $\epsilon_1,...,\epsilon_n, p_1,...,p_n$
\ENDIF
\end{algorithmic}

\end{algorithm}
% \subsubsection*{Bid transformation} 
% As shown in Algorithm \ref{alg:murba}, MURBA transforms each data owner's bid into $M$ sub-bids for $M$ items. That is, for each data owner $i$, given her reported bid $b_i'=(v_i', \bar{\epsilon}_i')$, MURBA computes $M$ sub-bids $b_i^{(1)},...,b_i^{(M)}$ and inputs them into our MBR model, where $b_i^{(m)}=(v_i'(\frac{\bar{\epsilon}_i'}{m}), \frac{\bar{\epsilon}_i'}{m}), \forall m\in[M]$. Then, MBR (see Figure \ref{fig:multi-regretnet}) outputs the allocation probabilities $z_{i1},...,z_{iM}$ of those $M$ items $\frac{\bar{\epsilon}_i'}{1},..., \frac{\bar{\epsilon}_i'}{M}$ for each $i$. Thus, the broker can always prepare $M+1$ nodes (i.e., $M$ for the valuations and one for the privacy budget) in the input layer for each data owner.  

To realize deterministic allocation results while ensuring trainable networks, we apply the soft argmax trick \cite{chapelle2010gradient} to DM-RegretNet. 
Then, as shown in Alg. \ref{alg:dm}, for the model inference phase, DM-RegretNet obtains deterministic allocation results by Equation \eqref{eq:one-hot}; for the model training phase, it uses the following differentiable estimator to approximate Equation \eqref{eq:one-hot}:
\begin{equation*}
    \hat{\epsilon}_i = [0, \frac{1 \cdot \bar{\epsilon}_i'}{M},...,\frac{M \cdot \bar{\epsilon}_i'}{M}] \cdot {softmax}(\boldsymbol{z_i}/\tau)
\end{equation*}
where $\tau$ is a smoothing parameter that controls the tradeoff between the estimator's approximation accuracy and smoothness. If we use a smaller $\tau$, the estimator $\hat{\epsilon}_i$ will approach the truth but become harder to optimize. 

Then, to further promote the approximation accuracy, we introduce the \textit{deterministic allocation} constraint when training DM-RegretNet, which requires that ${softmax}(\boldsymbol{z_i}/\tau)$ should be a one-hot vector.
Consider a vector $\boldsymbol{z}^{U} = [z^{U}_0,...,z^{U}_{M}]$ with uniform allocation probabilities, i.e., ${z}^{U}_{m} = \frac{1}{M+1}, \forall m \in [0, M]$.
Obviously, for a vector $\boldsymbol{z}=[z_{0},...,z_{M}]$ of allocation probabilities, the squared Euclidean distance between $\boldsymbol{z}$ and $\boldsymbol{z}^{U}$ is maximized only when $\boldsymbol{z}$ is a one-hot vector:
\begin{align*}
    \medmath{\sup_{\boldsymbol{z}} \sum_{m\in[0, M]}(z_{m}- z^{U}_m)^2 = (1-\frac{1}{M+1})^2 + M(0-\frac{1}{M+1})^2 = \frac{M}{M+1}}
\end{align*}
Then, we formalize the deterministic allocation constraint over the vector $\boldsymbol{z}_i'=[z_{i0}',...,z_{iM}']= {softmax}(\boldsymbol{z_i}/\tau)$ as:
\begin{equation*}
    {dav}_i(\theta)=\mathbb{E}_{(\boldsymbol{b}, B)} [ \frac{M}{M+1} - \sum_{m\in[0, M]}(z_{im}'- z^{U}_m)^2] = 0.
\end{equation*}
where $\theta$ is the network parameters of DM-RegretNet.
We note that $\boldsymbol{z}_i'$ is determined by the network parameters $\theta$ and the input $(b, B)$ to DM-RegretNet.

\subsubsection*{Training DM-RegretNet}
We train DM-RegretNet by solving Problem \ref{problem:auction}.
Concretely, given a (real) bid profile $\boldsymbol{b}$ and a financial budget $B$, we can obtain a global gradient:
\begin{equation*}
    \tilde{g}_{{\boldsymbol{\lambda}}, \theta}=\textsf{Aggr}(\boldsymbol{\hat{\epsilon}},\boldsymbol{d})\cdot[\mathcal{M}_{\hat{\epsilon}_1}(g_1),...,\mathcal{M}_{\hat{\epsilon}_n}(g_n)]
\end{equation*}
where the estimated privacy losses $\boldsymbol{\hat{\epsilon}}=[\hat{\epsilon}_1,...,\hat{\epsilon}_n]$ are affected by the network parameters $\theta$.
The training objective thus is to find the optimal network parameters that minimize the expected error bound $\mathbb{E}_{(\boldsymbol{b}, B)} [ERR(\tilde{g}_{{\boldsymbol{\lambda}}, \theta}; {\boldsymbol{\lambda}}=\textsf{Aggr}(\boldsymbol{\hat{\epsilon}},\boldsymbol{d}))]$.

Then, we relax the truthfulness constraint and quantify the violation degree of truthfulness for data owner $i$ by the \textit{expected regret} (normalized by the expected valuation of her allocated privacy loss $c_i^{\theta}(b_i; \boldsymbol{b}_{-i}, B)=\sum_{m=1}^{M} z_{im}' \cdot v_i(\frac{m\cdot \bar{\epsilon}_i}{M}, \bar{d}_i)$ under parameters $\theta$):
\begin{equation*}
    {rgt}_i(\theta) = \mathbb{E}_{(\boldsymbol{b}, B)}[\frac{\max(0, \max_{b_i'} u_i^{\theta}(b_i';\boldsymbol{b}_{-i}, B) - u_i^{\theta}(b_i; \boldsymbol{b}_{-i}, B))}{c_i^{\theta}(b_i; \boldsymbol{b}_{-i}, B)}]
\end{equation*}
where $u_i^{\theta}$ is data owner $i$'s utility function under network parameters $\theta$.
Similarly, the violation degree of the IR constraint can be measured by the \textit{expected IR violation}:
\begin{equation*}
  {irv}_i(\theta)=\mathbb{E}_{(\boldsymbol{b}, B)}[\frac{max(0, - u_i^{\theta}(b_i;\boldsymbol{b}_{-i}, B))}{c_i^{\theta}(b_i; \boldsymbol{b}_{-i}, B)}]
\end{equation*}
% where $\epsilon_i^\theta$ outputs owner $i$'s privacy loss under network parameters $\theta$.
Therefore, we have the following optimization problem.
% \vspace{-10pt}
\begin{problem}[DM-RegretNet Training Problem]
\label{problem:dm}
\begin{align*}
        &\min_{\theta} \mathbb{E}_{(\boldsymbol{b}, B)} [ERR(\tilde{g}_{{\boldsymbol{\lambda}}, \theta}; {\boldsymbol{\lambda}}=\textsf{Aggr}(\boldsymbol{\hat{\epsilon}},\boldsymbol{d}))]\\
    \text{S.t.: } & {rgt}_i(\theta) = 0, \forall i \quad \text{(Truthfulness)}\\ 
                  &  {irv}_i(\theta) = 0, \forall i \quad \text{(Individual Rationality)}\\ 
                  &  {dav}_i(\theta) = 0, \forall i \quad \text{(Deterministic Allocation)}
\end{align*}
\end{problem}

We can empirically estimate the expected error bound and those violation degrees from some training data and solve an empirical version of Problem \ref{problem:dm} to train DM-RegretNet.
The details can be checked in \venueforappendix{\ref{sec:dm-regret-train}}.
The training data can be drawn from a known distribution or historical data.
Note that DM-RegretNet is trained offline before the execution of Algorithm \ref{alg:framework}; 
in each FL training round, the trained auction model makes a model inference to decide the auction result, which efficiently finishes in polynomial time.

\begin{figure*}[t]
\vspace{-10pt}
\centering
\subfigure{
\begin{minipage}[t]{.5\linewidth}
\includegraphics[scale=0.3]{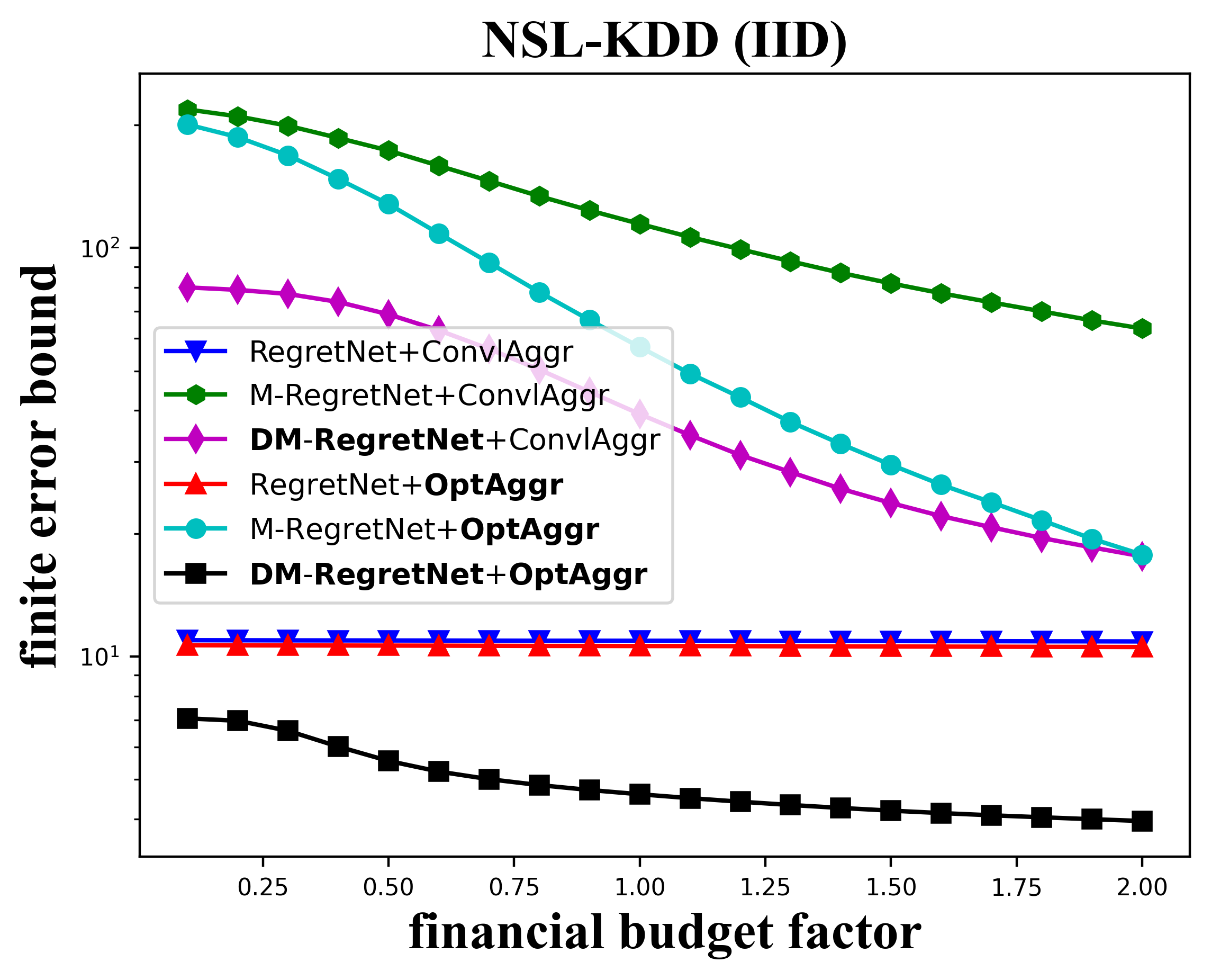}
\includegraphics[scale=0.3]{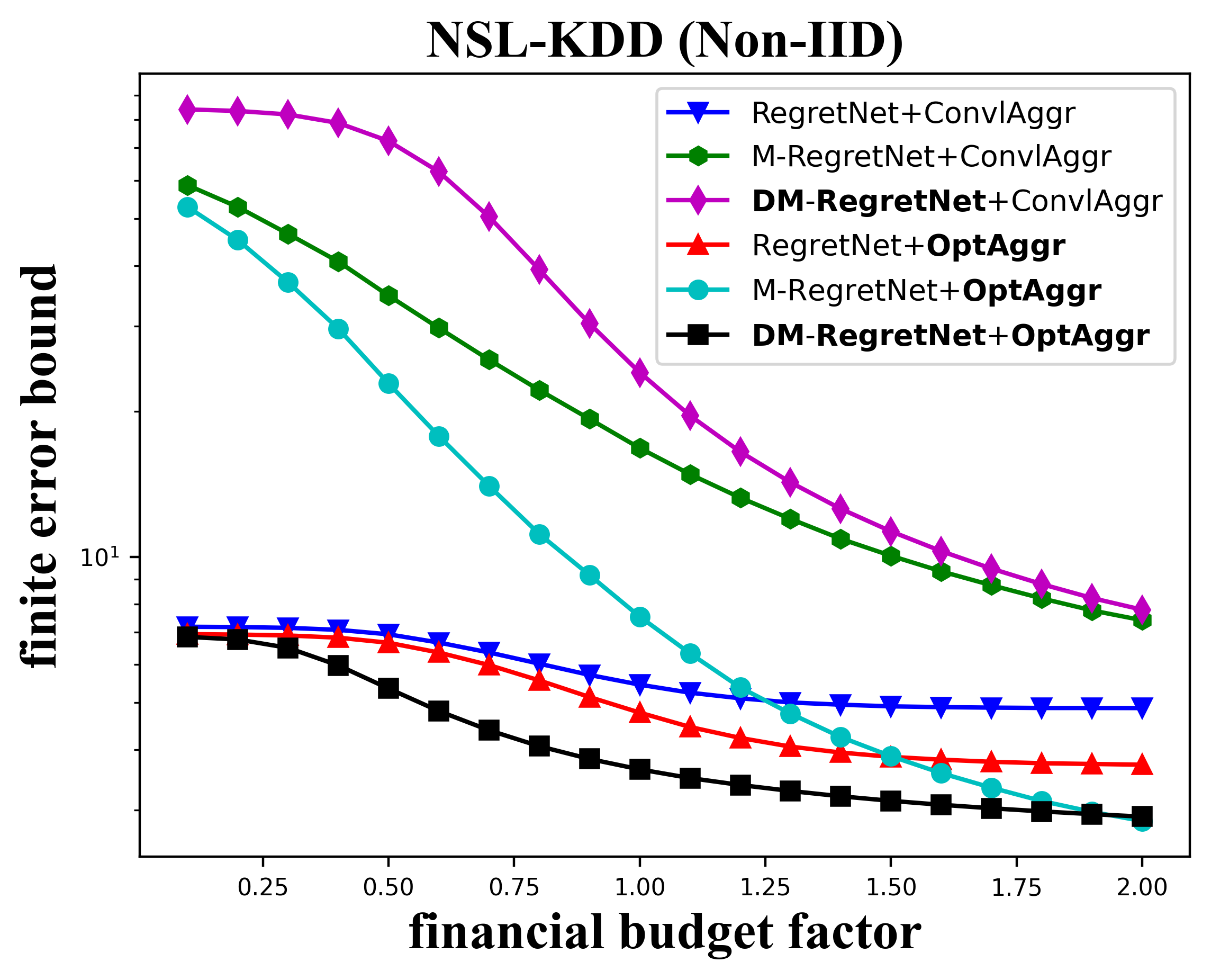}
\caption{Effect of financial budget on error bound.}
\label{fig:b_error}
\end{minipage}%
}%
% \vspace{-12pt}
% \vspace{-5pt}
% \vspace{-5pt}
\subfigure{
\begin{minipage}[t]{0.5\linewidth}
\includegraphics[scale=0.3]{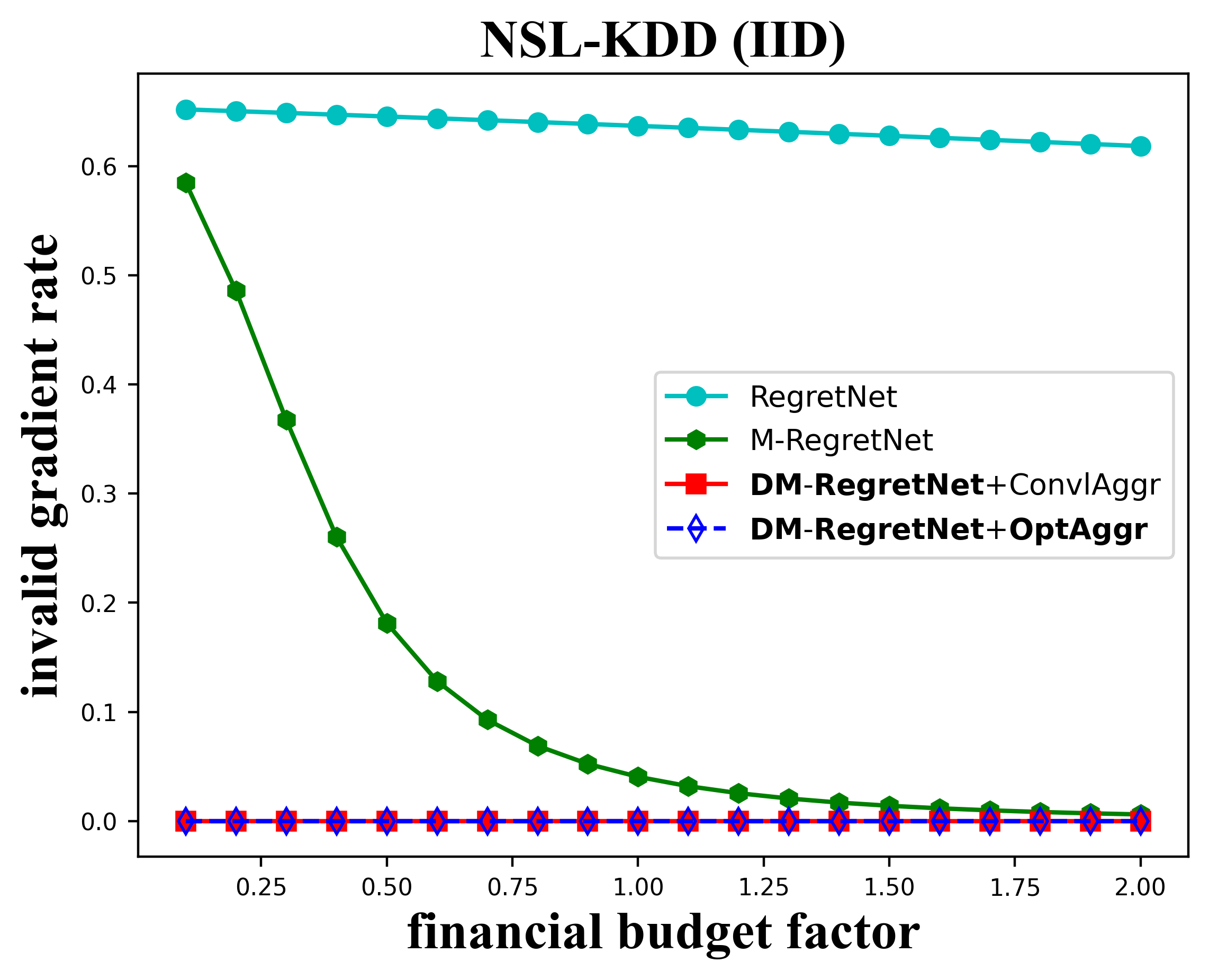}
\includegraphics[scale=0.3]{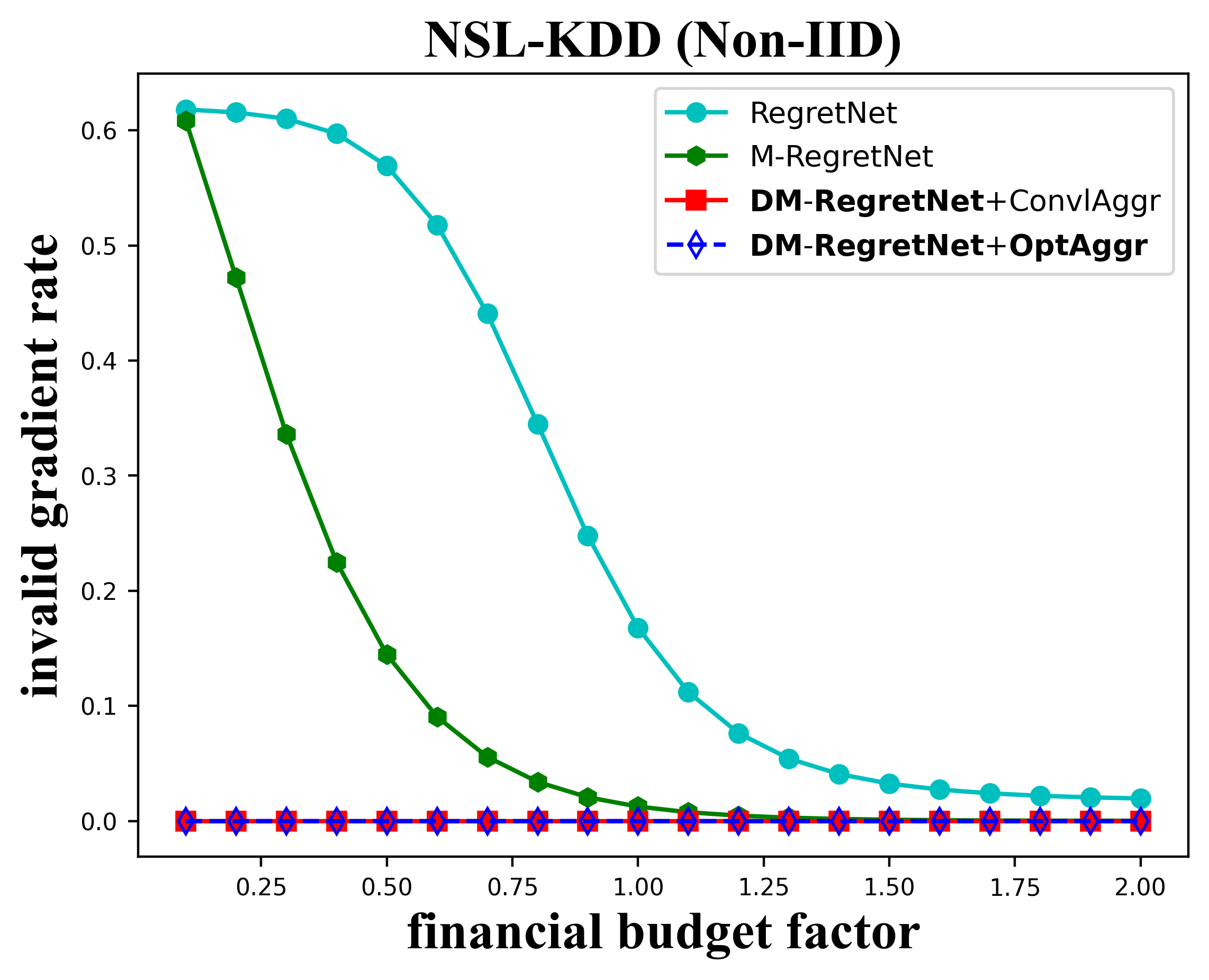}
\caption{Invalid gradient rate.}
\label{fig:b_invalid}
\end{minipage}%
}%
\vspace{-5pt}
\subfigure{
\begin{minipage}[t]{.5\linewidth}
% \includegraphics[scale=0.3]{fig/acc_single_bank_iid.png}
% \includegraphics[scale=0.3]{fig/acc_single_bank_niid.png}
% \includegraphics[scale=0.3]{fig/acc_single_nslkdd_iid.png}
% \includegraphics[scale=0.3]{fig/acc_single_nslkdd_niid.png}
% \caption{Effect of financial budget on model accuracy in single-minded cases.}
% \label{fig:b_acc_single}
% \end{minipage}%
% }%
% % \vspace{-5pt}
% \subfigure{
% \begin{minipage}[t]{.5\linewidth}
% \includegraphics[scale=0.3]{fig/acc_general_bank_iid.png}
% \includegraphics[scale=0.3]{fig/acc_general_bank_niid.png}
\includegraphics[scale=0.3]{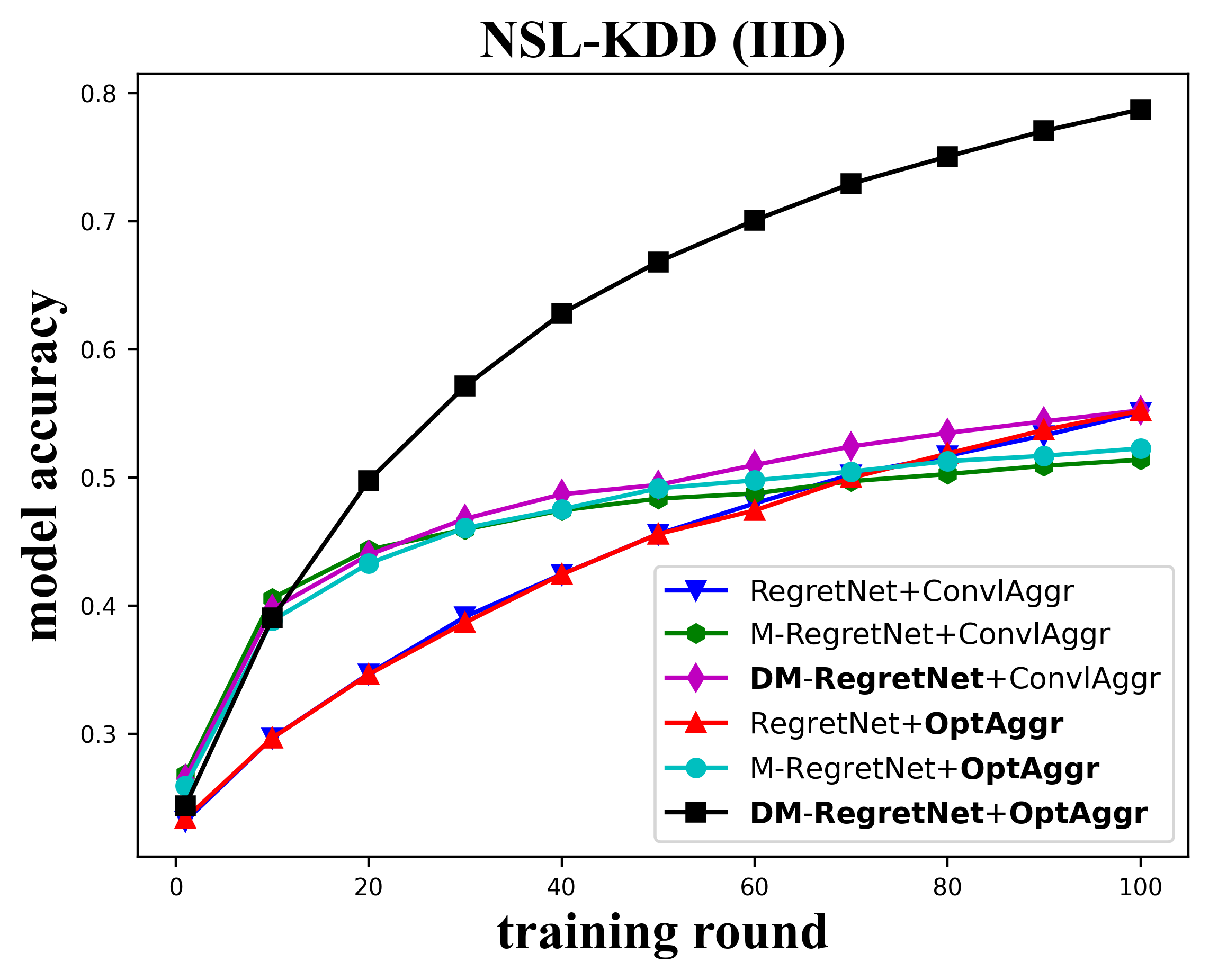}
\includegraphics[scale=0.3]{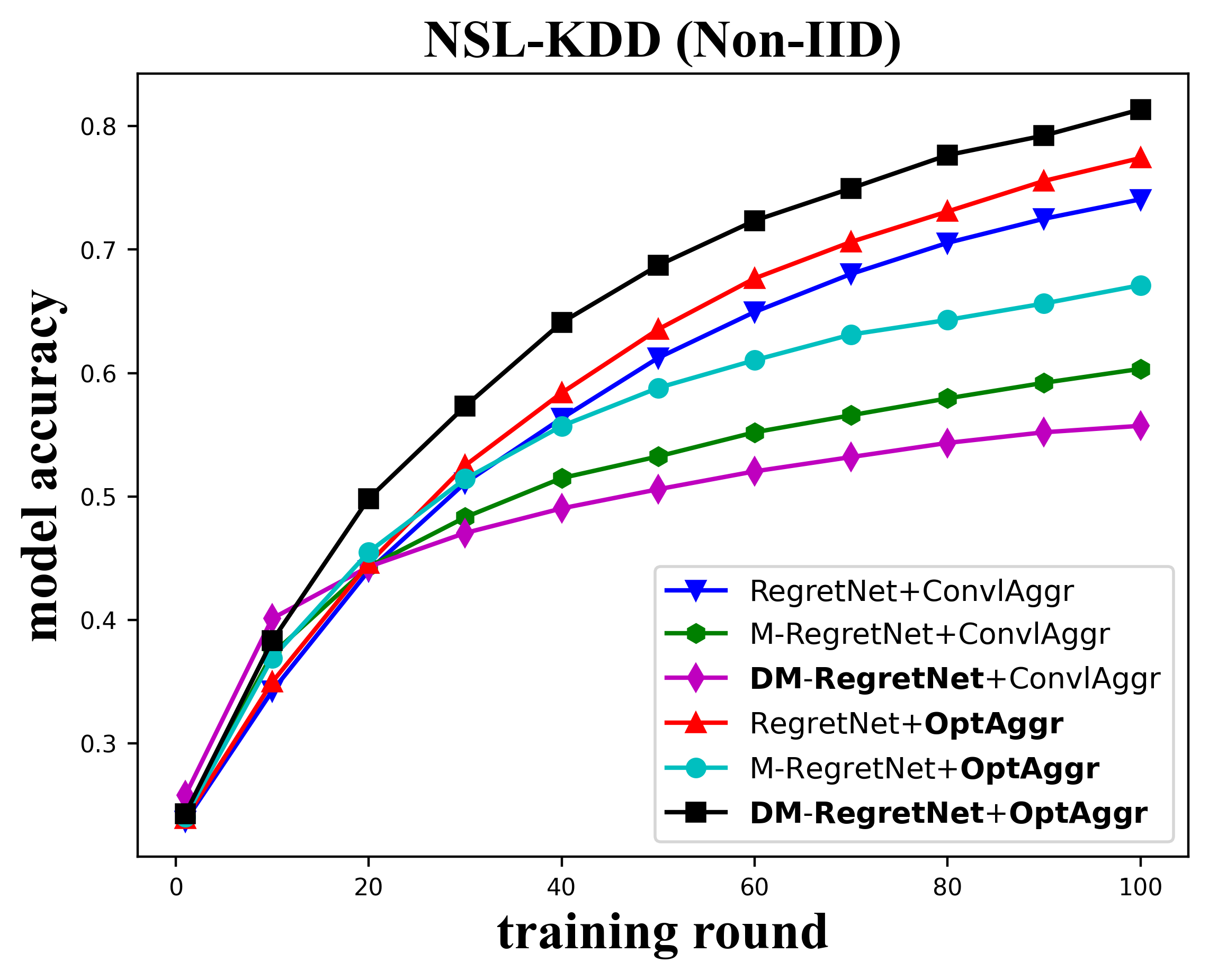}
% \caption{Effect of financial budget on model accuracy in general cases.}
\caption{Model accuracy over FL training rounds.}
\label{fig:b_acc}
\end{minipage}%
}%
\subfigure{
\begin{minipage}[t]{0.5\linewidth}
\includegraphics[scale=0.3]{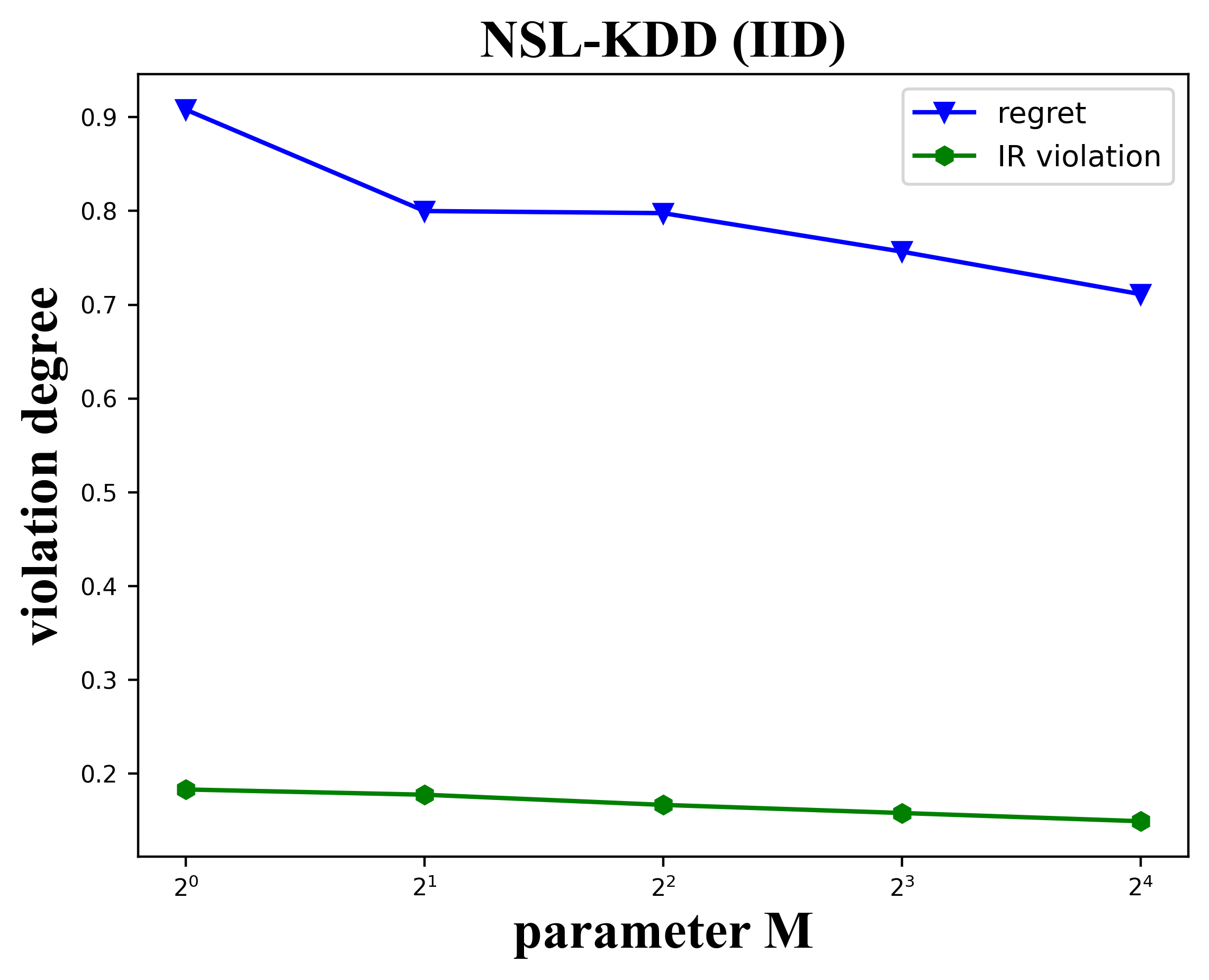}
\includegraphics[scale=0.3]{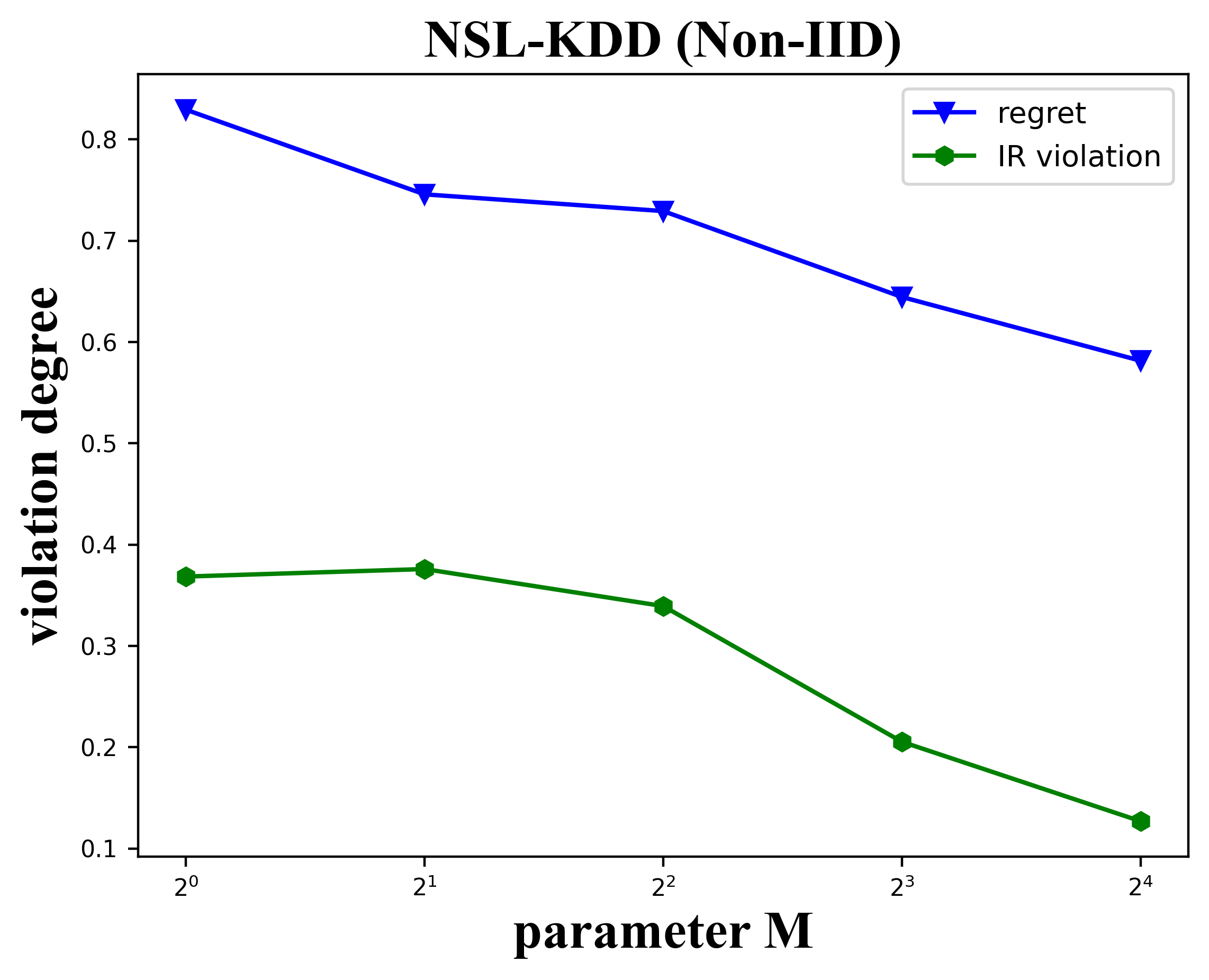}
\caption{Effect of parameter $M$.}
\label{fig:params}
\end{minipage}%
}%
% \vspace{-12pt}
% \vspace{-12pt}
\vspace{-10pt}
\end{figure*}

\section{Evaluation}
\label{sec:expr}
%In this section, we introduce the settings of our experiments and present some experimental results on our proposed mechanisms.

\subsection{Setup}
\subsubsection*{Research questions} 
We investigate the following research questions through experiments.
\begin{itemize}[leftmargin=*]
     \item RQ1: How does the proposed auction mechanism DM-RegretNet perform compared with the baselines (explained below) in terms of minimizing the error bound? 
     %How do they perform in multiple rounds of FL?
     %How do they perform when the training data are independent and identically distributed (IID) or when they are non-IID (NIID)?
     %How frequently do RegretNet and M-RegretNet sample zero-valued privacy losses for all data owners?
     \item RQ2: Can OptAggr outperform the conventional aggregation method in FL?
     \item RQ3: How does DM-RegretNet approximately guarantee the truthfulness and IR constraints?
    \item RQ4:
    % How does the number of data owners $n$ affect the performance of the auction mechanisms? 
    Does increasing $M$ benefit approximating the truthfulness and IR guarantees? 
 \end{itemize}
% We examine the above questions against baselines,  different types of training data distributions, and multiple evaluation metrics, as described below.

%Then, we group those questions into three packs: Pack1=(RQ1, RQ2, RQ5,RQ8), i.e., trading protocol comparison, Pack2=(RQ3, RQ7, RQ8), i.e., hyper-parameters' effects, Pack3=(RQ6, RQ7, RQ8), i.e., auction mechanism evaluation.
\subsubsection*{Baselines}
We compare OptAggr with the conventional aggregation method \textit{ConvlAggr} \cite{mcmahan2016federated}, which allocates positive aggregation weights only to those data owners with nonzero privacy losses, and the weights are proportional to their data sizes.
Regarding auction, we compare DM-RegretNet with RegretNet \cite{dutting2019optimal} and M-RegretNet.\footnote{Our code, data, and trained models are available at \url{https://github.com/teijyogen/FL-Market}.
We use the \textit{CVXPY} \cite{diamond2016cvxpy} and \textit{cvxpylayers} \cite{cvxpylayers2019} libraries to implement the OptAggr aggregation mechanism.
% We use the \textit{cvxpylayers} \cite{cvxpylayers2019} library to implement the OptAggr aggregation mechanism.
}
% For single-minded cases, the existing auction mechanism \textit{FairQuery} \cite{ghosh2015selling} with a simple adaption can serve as the benchmark.\footnote{The original version of FairQuery may allocate privacy losses beyond the privacy budgets.}
% For general cases, we compare DM-RegretNet with RegretNet \cite{dutting2019optimal} and M-RegretNet.

\subsubsection*{FL settings}
We use real data to train FL models. 
We choose logistic regression classifiers as FL models and use the NSL-KDD \cite{tavallaee2009detailed} datasets for $5$-class classification with $125973$ training samples and $22544$ test samples.
We distribute the training samples among $1000$ data owners to form their local datasets using the following partition methods:
\begin{itemize}[leftmargin=*]
    \item \textbf{IID}: We follow \cite{li2020federated} to draw all the local datasets from the same distribution, and their sizes follow a power law.
    \item \textbf{Non-IID}: We follow \cite{yurochkin2019bayesian} to allocate each class of samples among clients according to the Dirichlet distribution.
\end{itemize}
We set the learning rate $\eta = 0.01$ and the threshold $L=1.0$ for gradient clipping and perturb local gradients by the Laplace mechanism \cite{dwork2006calibrating}. 
% Due to the page limit, for some experiments, we only show the results on the NSL-KDD data if the other case has similar results.
% The gradients are perturbed by the Laplace mechanism \cite{dwork2014algorithmic}. 
% We also perform experiments on the BANK dataset \cite{moro2014data} in the full version \cite{zheng2021fl} and obtain similar results.

\subsubsection*{Auction settings}
% We synthesize bid profiles for auction.
% \footnote{Our code and data can be found in this anonymous link: \url{https://github.com/doubleblindcodes/flmarket}.}, which is generated in a similar way of RegretNet \cite{dutting2019optimal}.
For each run of the experiment, we simulate $100$ rounds of FL and generate $1000$ data owners; in each round, we randomly select $10$ data owners as bidders in the auction.
To simulate various types of bids, we let each bidder randomly select a basic valuation function from four provided: a linear function $v^L(\epsilon_i, d_i)=2\cdot d_i \cdot \epsilon_i$, a quadratic function $v^Q(\epsilon_i, d_i)=d_i \cdot (\epsilon_i)^2$, a square-root function $v^S(\epsilon_i, d_i)=2\cdot d_i \cdot \sqrt{\epsilon_i}$, and an exponential function $v^E(\epsilon_i, d_i)=d_i \cdot (\exp(\epsilon_i)-1)$, which are natural choices considered in \cite{ghosh2015selling}; these functions are directly proportional to the data size $d_i$ because it is natural to model the valuation of a dataset as the sum of the valuations of the data records therein.
Then, we consider each owner's valuation function to be a randomly selected rate $\alpha\in[0.5, 1.5]$ of the selected function, e.g., $v_i(\epsilon_i, d_i)=\alpha\cdot v^L(\epsilon_i, d_i)$. 
% For single-minded cases, we let $V_i = v_i(\bar{\epsilon}_i, \bar{d}_i)$ to model the single-minded valuations.
% Then, we follow Apple's privacy policy\footnote{\url{https://www.apple.com/privacy/docs/Differential_Privacy_Overview.pdf}} that sets privacy budgets in different sensitive scenarios. 
Finally, we randomly generate each data owner's privacy budget $\bar{\epsilon}_i \in [0.5, 2.0]$, which is in line with those commonly used in the differential privacy research community.
% \footnote{\url{https://www.apple.com/privacy/docs/Differential_Privacy_Overview.pdf}}
For DM-RegretNet and M-RegretNet, we set $M=8$ by default. 
We train all the auction models on $102,400$ bid profiles with $50$ epochs. 
% For all the experiments, we use a sample of $1,000$ bid profiles for testing. 

\subsubsection*{Evaluation metric}
To evaluate the utilities of the global gradients, we use as evaluation metrics the expected empirical error bound $\hat{ERR}$, the model accuracy (i.e., the percentage of correctly predicted examples), and the \textit{invalid gradient rate} (i.e., the frequency of sampling zero-valued privacy losses for all data owners).
Then, to evaluate the truthfulness and IR guarantees of the auction mechanisms, we use the empirical regret $\hat{rgt}_i$ and empirical IR violation $\hat{irv}_i$ as the metrics.
The definitions of the above metrics can be found in \venueforappendix{\ref{sec:dm-regret-train}}.
% Finally, we measure the extra error introduced by the allocation randomness of M-RegretNet by the \textit{introduced error}, which is the difference in the corresponding empirical error bound between the case in which data owners' privacy losses are random variables allocated by M-RegretNet and the case in which the privacy losses are equal to the expected values of those random variables.

% (defined in Appendix \ref{sec:model_training}).
% (defined in our anonymous technical report \cite{report}).
% (defined in our technical report).
% Due to the page limit, we will present the definitions of those metrics and also the settings of the hyper-parameters for training MBR in our technical report (to be released after acceptance).
% We present the definitions of those metrics and the settings of the hyper-parameters for training MBR in Appendix \ref{sec:model_training}.
% We note that no existing work can solve the incentivization and privacy protection problems simultaneously; 
% for the incentivization problem, existing incentive mechanisms for FL assume a trusted server, which is different to our setting;
% for the privacy protection problem, when personalized LDP is applied, it is also obvious that those non-incentive locally differentially-private FL sysmtes (e.g., \cite{liu2020fedsel, sun2020ldp}) can benefit from adopting our proposed aggregation mechanisms to replace the conventional data size-weighted averaging method \cite{mcmahan2016federated}.
% Therefore, we only compare our proposed mechanisms.

\subsection{Experimental Results}
% We highlight the labels of our proposed (advanced) mechanisms in the figures that present experimental results.

\subsubsection*{Incentive mechanisms comparison (RQ1)}
First, we test the auction mechanisms' performance in minimizing the error bound.
We vary the financial budget factor $\bar{B}$ and let the budget $B=\bar{B}\cdot \sum_{i\in[n]} v_i(\bar{\epsilon}_i, \bar{d}_i)$.
As shown in Figure \ref{fig:b_error}, our DM-RegretNet can generate global gradients with a lower error bound in expectation.
% As shown in Figure \ref{fig:b_error}, under each aggregation mechanism, our DM-RegretNet mechanism can generate global gradients with a lower error bound in expectation.
When the financial budget factor increases and exceeds $1.0$, which means the budget covers the gross valuation of the bidders' privacy budgets, the error bound may still be able to decrease since the payments made by a truthful auction mechanism are usually much higher than the winners' valuations.
We note that since the randomized mechanisms RegretNet and M-RegretNet may sample zero-valued privacy losses for all data owners, which results in \textit{invalid global gradients} with infinite error, we only take the error bound of valid gradients into account.
That means that even if Figure \ref{fig:b_error} shows that RegretNet results in low error bounds, it actually frequently generates invalid gradients with infinite error, while our DM-RegretNet based mechanisms never do, which is depicted in Figure \ref{fig:b_invalid}.
For the rest experiments, we sample the budget factor uniformly at random from $[0.1, 2.0]$.
We also test the model accuracy over $100$ FL training rounds.
As shown in Figure \ref{fig:b_acc}, in both cases, DM-RegretNet makes better auction decisions that result in more accurate models.

% As shown in Figure \ref{fig:b_error}, when the financial budget is relatively small, MURBA alwaysoutperforms All-in since the curves of All-in disappear. That is because All-in only allows allocating data owners' whole privacy budgets as their privacy losses, and even no one is selected when the financial budget is highly limited. We can also observe that VarOpt performs better than BiasOpt. In addition, in the highly sensitive scenario, All-in is superior when the financial budget is relatively large, because data owners' privacy budgets are small enough to purchase. 

\subsubsection*{Aggregation mechanisms comparison (RQ2)}
As depicted in Figure \ref{fig:b_error}, under each auction mechanism, our OptAggr aggregation mechanism can always generate global gradients with a lower error bound in expectation than ConvlAggr.
% FairQuery performs as an exception because it always outputs uniform privacy losses for all data owners; when privacy losses are uniform, ConvlAggr is equivalent to OptAggr.
In addition, Figure \ref{fig:b_acc} shows that model buyers can obtain more accurate models using global gradients aggregated by OptAggr.
Therefore, OptAggr outperforms ConvlAggr.

\begin{table}[t]
\small
    \centering
    \caption{Comparisons of the violation degrees of truthfulness and IR of RegretNet-based mechanisms. At each box, the two numbers are the empirical regret and IR violation, respectively.}
    \begin{tabular}{|l|l|l|}
    \hline
        ~ & IID & Non-IID \\ \hline
        RegretNet & 0.9351, 0.1684 & 0.8164, 0.3864 \\ \hline
        M-RegretNet & 0.7715, 0.1508 & 0.6652, 0.2020 \\ \hline
        DM-RegretNet+ConvlAggr & 0.0617, 0.0251 & 0.0516, 0.0210 \\ \hline
        DM-RegretNet+OptAggr & 0.0556, 0.0265 & 0.0428, 0.0259 \\ \hline
    \end{tabular}
    \label{table:guarantees}
    \vspace{-3pt}
\end{table}

\subsubsection*{Incentive guarantees (RQ3).}
Table \ref{table:guarantees} illustrates the violation degrees of truthfulness and IR of those RegretNet-based auction mechanisms.
% We can see that the normalized regret and the normalized IR violation of DM-RegretNet are much lower than those of RegretNet, which means DM-RegretNet can approximate the truthfulness and IR constraints better than RegretNet.
The empirical regrets and IR violations under DM-RegretNet are significantly lower than those under RegretNet and M-RegretNet, which means that DM-RegretNet has stronger abilities to approximate the truthfulness and IR constraints.
DM-RegretNet has this advantage because it is a deterministic mechanism that universally guarantees truthfulness and IR, while RegretNet and M-RegretNet are randomized mechanisms that approximate the two constraints by expectation.
% In addition, M-RegretNet has the lowest normalized regret and the lowest normalized IR violation, probably because M-RegretNet can flexibly allocate privacy losses of more possible values than RegretNet and does not impose a deterministic allocation constraint in contrast to DM-RegretNet. 

\subsubsection*{Parameter effects (RQ4)}
% When testing the parameter effects, we also randomly select budget factor $\bar{B}\in[0.1, 1.0]$ for each auction.
We vary the value of parameter $M \in \{1, 2, 4, 8, 16\}$ to test its effects on the truthfulness and IR guarantees.
For each value, we train $10$ instances of M-RegretNet and test them to obtain the average result.
Figure \ref{fig:params} shows that under M-RegretNet, an increase in $M$ decreases both the regret and IR violation, which demonstrates our intuition that a larger $M$ can enhance the abilities of M-RegretNet to approximate the truthfulness and IR guarantees as it has more possible values to allocate as privacy losses.
\section{Related Work}
\label{sec:related_work}

\subsubsection*{Incentive mechanisms for FL}
% Plenty of incentive mechanisms have been proposed to encourage data owners' participation in FL. Those incentive mechanisms incentivize data owners by providing appropriate rewards and can be classified into four categories according to their rewarding approaches: rewarding data owners based on (1) the size of data samples they contribute \cite{jiao2020toward, zhan2020learning}, (2) the quality of data samples \cite{jiao2020toward, richardson2019rewarding, wang2020principled}, (3) their computation and communication costs \cite{sarikaya2019motivating, jiao2020toward, kang2019incentive}, or (4) their reputation \cite{kang2020reliable, xu2020towards}. However, none of them considers privacy protection. They allow the transactions of raw gradients for FL and assume a trusted server, which may reduce data owners' willingness to participate in FL. Therefore, in this paper, we propose a privacy-preserving incentive mechanism that enables data owners to trade private gradients under the protection of LDP.
Many incentive mechanisms \cite{jiao2020toward, richardson2019rewarding, wang2020principled, sarikaya2019motivating, kang2019incentive, kang2020reliable, zhan2020learning, zhang2021incentive} have been proposed to encourage participation in FL by providing appropriate rewards for data owners' contributions. 
The contributions can be evaluated in various ways.
For example, Zhan et al. \cite{zhan2020learning} consider the data size, the most basic measurement of data, for contribution evaluation. 
Then, from a cost-covering perspective, Jiao et al. \cite{jiao2020toward} propose an auction mechanism where data owners can bid their computational and communication costs in providing their FL training services.
Similarly, Sarikayar et al. \cite{sarikaya2019motivating} regard the CPU computational costs as their contributions.
Then, Richardson et al. \cite{richardson2019rewarding} evaluate data owners' influences on the model accuracy to decide their rewards. 
The Shapley value is also adapted into an FL version by Wang et al. \cite{wang2020principled} to value data owners' influence. 
Data quality is another natural choice.
Since the data quality is known only to data owners, to ensure the contribution of high-quality data, Kang et al. \cite{kang2019incentive} design different types of rewarding contracts to distinguish data owners such that the FL server can infer the data quality based on the contracts they select.
In this way, the rewards are essentially determined by the data quality.
In addition, both the works of Kang et al. \cite{kang2020reliable} and Zhang et al. \cite{zhang2021incentive} employ some reputation metric to remove unreliable data owners from FL.
However, none of the above mechanisms considers privacy protection, which is also a critical incentive. 
To fill this gap, we propose an auction-based incentive mechanism that protects data owners' privacy and compensates them according to their privacy preferences.
%To fill this gap, we propose an auction-based incentive mechanism that allows each data owner to bid her privacy budget and privacy valuation.

% \vspace{-6pt}
\subsubsection*{FL under LDP}
Some efforts \cite{li2019differentially, wang2020federated, sun2020ldp, liu2020fedsel, liu2020flame, zhao2020local, wu2022fedctr, girgis21shuffled} have devoted to designing FL frameworks under LDP. 
% Some efforts \cite{sun2020ldp, liu2020fedsel, liu2020flame} have been put on designing federated learning frameworks under LDP. 
Since the data perturbation under LDP may substantially reduce the utility of FL models, these authors mainly focus on how to reduce the perturbation level while still providing appropriate privacy guarantees. 
Concretely, to relieve the utility problem that the noise that LDP injects into a gradient should be proportional to its size, Liu et al. \cite{liu2020fedsel} propose an FL framework to perturb only the top-k important dimensions of the gradient and thus can its utility. 
Then, Liu et al. \cite{liu2020flame} and Girgis et al. \cite{girgis21shuffled} employ the shuffle model \cite{erlingsson19ampli} in their FL frameworks to amplify the privacy guarantee under the same level of noise injection. 
Then, Sun et al. \cite{sun2020ldp} propose a more secure LDP mechanism that can extend the difference between the perturbed data and its original value while introducing lower variance. 
There are also works on designing LDP-based FL frameworks for specific ML tasks \cite{li2019differentially, wang2020federated, wu2022fedctr}. 
% For example, Li et al. \cite{li2019differentially} proposed a federated meta-learning framework with a relaxed LDP guarantee.
% Wang et al. \cite{wang2020federated} designed an FL mechanism for training Latent Dirichlet Allocation models under LDP.
% Wu et al. \cite{wu2020fedctr} also proposed an FL method for mining user interest from distributed user behavior data in a locally private manner. 
While prior works address the utility problem under LDP by relaxing the privacy guarantee or elaborately injecting noise, we tackle it from an incentive perspective, i.e., by incentivizing data owners to contribute more privacy loss, which can also increase utility.
In addition, Zhao et al. \cite{zhao2020local} propose an LDP-based FedSGD algorithm, which is similar to our protocol in privacy protection; however, they assume uniform privacy losses for all data owners and thus do not consider different perturbation levels when aggregating gradients.

\section{Conclusion and Future Work}
In this paper, we propose FL-Market to facilitate trustworthy data acquisition for ML-based data analytics.
Our mechanisms can incentivize data sharing by providing preferred levels of local privacy and compensation for data owners and optimizing model buyers' utility.
% FL-Market opens new possibilities for incentive mechanism design and new ideas to improve the utility of differentially private FL towards a closed-loop FL system.
FL-Market opens up new possibilities for ML-oriented data acquisition and initiates a new direction toward designing locally private model marketplaces.
There are several interesting future directions.
One question is how to guarantee that the auction decisions are arbitrage free against strategic buyers.
Another question is how to apply and optimize FL-Market in specific learning tasks. 

%FL-Market incentivizes data owners' participation in FL by providing appropriate payments with them, and improve the FL model utility under LDP by minimizing the error bound of the perturbed global gradient. 
%Then, we proposed two aggregation mechanisms to aggregate perturbed local gradients in an optimal way, and two truthful auction mechanisms that can properly decide payments. 
%Finally, the experimental results verify the truthfulness and effectiveness of the proposed mechanisms.
% Finally, we conducted comprehensive experiments on FL-Market. 
% The experimental results verified the effectiveness of our trading protocols in terms of incentive guarantees and minimizing the error bound.

\label{sec:conclu}

\section{Acknowledgment}
We thank the Japan Society for the Promotion of Science (JSPS) for its generous and continued support for the first author who conducted this research as a JSPS Research Fellow.
In addition, this work was partially supported by JST CREST (No. JPMJCR21M2), JST SICORP (No. JPMJSC2107), and JSPS KAKENHI (No. 21J23090, 21K19767, 22H03595).

\bibliographystyle{IEEEtran}
\bibliography{ref.bib}

\normalsize
\showappendix{\section{Training DM-RegretNet}
\label{sec:dm-regret-train}
 Consider a training sample $\mathcal{S}=(S^1,...,S^T)$ consisting of $T$ batches. Each batch $S^{t}=((b^{(1))}, B^{(1)}),...,(b^{(K)},B^{(K)})), t\in[T]$ has $K$ pairs of real bid profiles and financial budgets, and each profile $b^{(k)}, k\in[K]$ consists of a valuation function $v_i^{(k)}$, a privacy budget $\bar{\epsilon}_i^{(k)}$, and a data size $\bar{d}_i^{(k)}$. 
 Then, at each training iteration $t$, we can estimate $rgt_i(\theta^t)$ by the \textit{empirical regret}:

 \begin{equation*}
	     \hat{rgt}_i(\theta^t)\\
	     = \frac{1}{K}\sum_{k=1}^{K} \frac{max(0, u_i^{\theta^t}(b_i^{* (k)};\boldsymbol{b}_{-i}^{(k)}, B^{(k)}) - u_i^{\theta^t}(b_i^{(k)};\boldsymbol{b}_{-i}^{(k)}, B^{(k)}))}{c_i^{\theta^t}(b_i^{(k)}; \boldsymbol{b}_{-i}^{(k)}, B^{(k)})}
	 \end{equation*}
 % \begin{align*}
	 %     &\hat{rgt}_i(\theta^t, S^t)\\
	 %     = &\frac{1}{K}\sum_{k=1}^{K} max(0, u_i^{\theta^t}(b_i^{* (k)};\boldsymbol{b}_{-i}^{(k)}, B^{(k)}) - u_i^{\theta^t}(b_i^{(k)};\boldsymbol{b}_{-i}^{(k)}, B^{(k)}))
	 % \end{align*}
 where $\theta^t$ represents the network parameters at training iteration $t$ and $b_i^{* (k)}$  is a bid that approximately maximizes $i$'s utility and is searched through $J$ updates of the following optimization process:
 \begin{equation*}
	 b_i'^{(k)} \gets b_i'^{(k)} + \gamma \nabla_{b_i'} u_i^{\theta^t}(b_i';\boldsymbol{b}_{-i}^{(k)}, B^{(k)})|_{b_i'=b_i'^{(k)}}
	 \end{equation*}
 Similarly, we estimate $irv_i(\theta^t)$ by the \textit{empirical IR violation}:
 \begin{equation*}
	   \hat{irv}_i(\theta^t)=\frac{1}{K}\sum_{k=1}^{K} \frac{max(0, - u_i^{\theta^t}(b_i^{(k)};\boldsymbol{b}_{-i}^{(k)}, B^{(k)}))}{c_i^{\theta^t}(b_i^{(k)}; \boldsymbol{b}_{-i}^{(k)}, B^{(k)})}
	 \end{equation*}

 Let $[z^{(k)}_{i0},...,z^{(k)}_{iM}]$ denote the allocation probabilities for data owner $i$ given bid profile $b^{(k)}$ and financial budget $B^{(k)}$ under network parameters $\theta^t$, and let $[z'^{(k)}_{i0},...,z'^{(k)}_{iM}]={softmax}(\frac{[z^{(k)}_{i0},...,z^{(k)}_{iM}]}{\tau})$. Then, we have the \textit{empirical deterministic allocation violation} $\hat{dav}_i(\theta^t)$ to estimate $dav_i(\theta^t)$:
 \begin{equation*}
	     \hat{dav}_i(\theta^t) = \frac{1}{K}\sum_{k=1}^{K}  [\frac{M}{M+1} -  \sum_{m\in[0, M]}(z'^{(k)}_{im}- \frac{1}{M+1})^2]
	 \end{equation*}

 Finally, we should derive an empirical version of the expected error bound $\mathbb{E}_{(b, B)} [ERR(\tilde{g}_{{\boldsymbol{\lambda}}, \theta}; {\boldsymbol{\lambda}}=\textsf{Aggr}(\boldsymbol{\hat{\epsilon}},\boldsymbol{d}))]$.
 Let $\hat{\epsilon}_{1}^{(k)},...,\hat{\epsilon}_{n}^{(k)}$ denote the estimated privacy losses determined by DM-RegretNet for bid profile $b^{(k)}$ and financial budget $B^{(k)}$. 
 Given aggregation weights $\hat{\lambda}_1^{(k)},...,\hat{\lambda}_n^{(k)}=\textsf{Aggr}([\hat{\epsilon}_{1}^{(k)},...,\hat{\epsilon}_{n}^{(k)}], [\bar{d}_1^{(k)},...,\bar{d}_n^{(k)}])$ and $W_i^{(k)} = \frac{\bar{d}_i^{(k)}}{\sum_{j\in[n]} \bar{d}_j^{(k)}}, \forall i$, we have the \textit{empirical expected error bound}:
 \begin{equation*}
	     \hat{ERR}(\theta^t) =\frac{1}{K}\sum_{k=1}^{K} \sup_{g_1,...,g_n}||\sum_{i=1}^{n}\hat{\lambda}_i^{(k)} \cdot\mathcal{L}_{\hat{\epsilon}_i^{(k)}}^{L}(g_i)-\sum_{i=1}^{n}W_i^{(k)} g_i||_2
	 \end{equation*}
 % Then, we can solve our error bound minimization problem while guarantee the truthfulness and IR of MURBA by solving the following optimization problem:
 % \begin{align*}
	 %     &\min_{\theta^t} \hat{ERR}(\theta^t, S^t) \quad (P3)\\
	 %     \text{S.t.: } & \hat{rgt}_i(\theta^t, S^t) = 0, \forall i\\
	 %                   & \hat{irv}_i(\theta^t, S^t) = 0, \forall i
	 % \end{align*}

 We can solve Problem \ref{problem:dm} by the augmented Lagrangian method and minimize the following Lagrangian function:
 \footnote{When training RegretNet and M-RegretNet, we minimize 
	 the negated empirical privacy loss $\hat{NPL}(\theta^t)=-\frac{1}{K\cdot n}\sum_{k=1}^{K} \sum_{i=1}^{n} W_i^{(k)} \mathbb{E}[\epsilon_i^{(k)}]$ instead of $\hat{ERR}(\theta^t)$, where $\mathbb{E}[\epsilon_i^{(k)}]$ is the expected privacy loss of the $i$-th data owner of the $k$-th bid profile at the $t$-th batch.}
 % $\hat{ERR}(\theta^t; \mathbb{E}[\epsilon_1^{t,k}],...,\mathbb{E}[\epsilon_n^{t,k}])$, where $\epsilon_1^{t,k},...,\epsilon_n^{t,k}$ are its random allocation results for the $k$-th bid profile of the $t$-th batch.}
 \begin{align*}
     &\mathcal{C}(\theta^t;\phi_{rgt}^t, \phi_{irv}^t, \phi_{dav}^t) \\
     =& n\cdot \hat{ERR}(\theta^t) + \sum_{i=1}^{n} \phi_{rgt,i}^t \cdot \hat{rgt}_i(\theta^t) + \frac{\rho_{rgt}}{2}(\sum_{i=1}^{n} \hat{rgt}_i(\theta^t))^2 \\
     &\quad\quad\quad\quad+ \sum_{i=1}^{n} \phi_{irv,i}^t \cdot \hat{irv}_i(\theta^t) + \frac{\rho_{irv}}{2}(\sum_{i=1}^{n} \hat{irv}_i(\theta^t))^2 \\
     &\quad\quad\quad\quad+ \sum_{i=1}^{n} \phi_{dav,i}^t \cdot \hat{dav}_i(\theta^t) + \frac{\rho_{dav}}{2}(\sum_{i=1}^{n} \hat{dav}_i(\theta^t))^2
 \end{align*}
 where $\phi_{rgt}^t, \phi_{irv}^t, \phi_{dav}^t \in R^n$ are vectors of Lagrange multipliers and $\rho_{rgt},\rho_{irv},\rho_{dav} >0$ are fixed hyperparameters that control the quadratic penalties.
 Finally, the network parameters of DM-RegretNet are updated at each iteration $t$ as:
 \begin{equation*}
     \theta^{t+1} \gets \theta^t - \psi \nabla_{\theta} \mathcal{C}(\theta^t;\phi_{rgt}^t, \phi_{irv}^t, \phi_{dav}^t)
 \end{equation*}
 and the Lagrange multipliers are updated every $Q$ iterations as:
 \begin{align*}
         \text{If } t\bmod{Q}=0:\forall i, &\phi_{rgt,i}^{t+1} \gets \phi_{rgt,i}^t + \rho_{rgt} \cdot\hat{rgt}_i(\theta^t) \\
         & \phi_{irv,i}^{t+1} \gets \phi_{irv,i}^t + \rho_{irv} \cdot\hat{irv}_i(\theta^t) \\
         & \phi_{dav,i}^{t+1} \gets \phi_{dav,i}^t + \rho_{dav} \cdot \hat{dav}_i(\theta^t)
 \end{align*}

 In our experiments, we fine-tune and set the hyperparameters as follows: $T=100$, $K=1024$, $J=100$, $Q=10$, $\gamma=0.1$, $\psi = 0.001$, and $\phi_{rgt,i}^{1}=\phi_{irv,i}^{1}=\phi_{dav,i}^{1}=1.0$; the allocation (payment) network consists of $2$ hidden layers and $100$ hidden nodes per layer. We train each model for $50$ epochs. In addition, we set $\rho_{rgt}=\rho_{irv}=\rho_{dav}=1.0$ at the first epoch of training and increase $\rho_{rgt}, \rho_{irv}$ in steps of $1.0$ at the end of every epoch.
 We note that since we only need the bid profiles and financial budgets to train DM-RegretNet, which are assumed to be nonprivate, fine-tuning the hyperparameters of DM-RegretNet does not cause any privacy leakage. 
}

\section{All-in: Single-Minded Auction Mechanism}
\label{sec:all-in}
We propose an auction mechanism All-in for \textit{single-minded} data owners, each of whom has a step valuation function $v_i(\epsilon_i, d_i)=\begin{cases}
	V_i, & \epsilon_i \in (0, \bar{\epsilon}_i], d_i \in (0,\bar{d}_i]\\
	0, & \epsilon_i = 0 \text{ or } d_i = 0
\end{cases}$ where $V_i>0$ is a constant set by $i$. 
Therefore, we can use $V_i$ and $V_i'$ to represent the real valuation $v_i$ and the reported valuation $v_i'$, respectively.
Such cases are common in practice because some data owners are just willing to sell all their small datasets and privacy budgets at a single round of auction or only focus on whether their private information is leaked rather than how much is leaked.
Obviously, each data owner $i$ can only have two kinds of auction results: (1) win the auction with $\epsilon_i = \bar{\epsilon}_i'$ or (2) lose the auction with $\epsilon_i=0$. 

To meet the demands of single-minded bidders, we can design a truthful mechanism using Myerson's characterization \cite{myerson1981optimal}, which indicates that the \textit{monotonicity} and \textit{critical payment} properties imply truthfulness. 
Concretely, monotonicity requires that a winner should still win if she re-reports a higher privacy budget, a larger data size, and/or a lower valuation with other bidders' bids fixed;
the critical payment property ensures that winners are paid the maximum possible payments (i.e., critical payments) and hence that they have no incentive to misreport bids.  
However, the limited financial budget makes the problem more difficult because the winner selection should depend on the payments, which in turn depend on the selection results.
Hence, we should carefully identify budget-feasible critical payments.

\begin{algorithm}[h]
	% \small
	\caption{Auction Mech.: All-in}
	\begin{algorithmic}[1]
		\REQUIRE (reported) bid profile $b'=(b_1',...,b_n')$, financial budget $B$
		\ENSURE privacy losses $\epsilon_1,..,\epsilon_n$, payments $p_1,...,p_n$
		\STATE Calculate the unit valuations on privacy budgets: $\forall i, v_i^{unit}=\frac{V_i'}{d_i\cdot \bar{\epsilon}_i'}$
		\STATE Sort data owners in ascending order of $v_i^{unit}$
		\STATE Initialize the winner set $\mathcal{W} = \emptyset$ and critical unit payment $p^{unit}=0$
		\FOR{each data owner $i$ in the sorted order}
		\STATE If $v_i^{unit} \leq \frac{B}{\sum_{j\in\mathcal{W}\cup \{i\}} d_j \cdot \bar{\epsilon}_j'}$, add $i$ into $\mathcal{W}$ and update critical unit payment $p^{unit}=\frac{B}{\sum_{j\in\mathcal{W}\cup \{i\}} d_j \cdot \bar{\epsilon}_j'}$
		% \IF{$\sum_{j\in\mathcal{W}\cup \{i\}} \bar{\epsilon}_i' \cdot v_i^{unit} \leq B$}
		%     \STATE Add $i$ into $\mathcal{W}$
		% \ENDIF
		\ENDFOR
		\STATE Calculate privacy losses: $\forall i, \epsilon_i=\bar{\epsilon}_i'$ if $i\in\mathcal{W}$; otherwise $\epsilon_i=0$ 
		% \STATE Calculate the critical unit payment: $p^{unit}=\frac{B}{\sum_{j\in\mathcal{W}} \bar{\epsilon}_i}$ if $\mathcal{W} \neq \emptyset$; otherwise $p^{unit}=0$
		\STATE Calculate payments: $\forall i, p_i = d_i \cdot \epsilon_i \cdot p^{unit}$
		\RETURN $\epsilon_1,...,\epsilon_n, p_1,...,p_n$
	\end{algorithmic}
	\label{alg:all_in}
\end{algorithm}

% In Theorem \ref{thm:myerson}, the monotonicity property prevents the cases where a data owner, who is not able to win the auction by truthfully reporting her real bid, can win by providing a higher privacy loss or underreport the valuation. 
% In Theorem \ref{thm:myerson}, the monotonicity property guarantees that a winning data owner will still win the same auction if she re-reports a higher privacy budget or a lower valuation.
% On the other hand, the critical payment property ensures a data owner, who can win the auction with her real bid, will not obtain a higher payment or even will lose the auction if she untruthfully reports a higher valuation. Based on these two properties, we propose the truthful single-minded auction mechanism All-in.

% In the single-minded case, each data owner $i$ can only have two kinds of auction results: (1) to have a privacy loss equal to her reported privacy budget, i.e., $\epsilon_i = \bar{\epsilon}_i'$, or (2) to lose the auction with $\epsilon_i=0$. 

To capture the interdependency between the winner selection and payment decision, All-in takes the payments into account when selecting winners. 
Concretely,  to guarantee monotonicity, All-in selects data owners in ascending order of their \textit{unit valuations} $ v_i^{unit}=\frac{V_i'}{d_i\cdot \bar{\epsilon}_i'}$; 
intuitively, if a owner $i$ decreases her valuation $V_i'$, increases her data size $d_i$, and/or increases her privacy budget $\bar{\epsilon}_i’$, she stays at the same position or moves to a former position in the order.
Then, the winner selection procedure is to find the last owner whose unit valuation $v_i^{unit}$ is covered by the critical unit price $\frac{B}{\sum_{j\in\mathcal{W}\cup \{i\}} d_j \cdot \bar{\epsilon}_j'}$.
In this design, the winners' payments that exhaust the financial budget $B$ are critical because if a winner $i$ claims a higher unit valuation $v_i^{unit'} > p^{unit}$ to gain a higher payment, she definitely loses the auction due to the violation of BF.
Therefore, truthfulness is ensured.
% Then, All-in pays winners at the same critical unit price $p^{unit}$, 
% which results in payments $\epsilon_i\cdot p^{unit}, \forall i \in \mathcal{W}$ that exhaust the budget $B$; 

% To involve more privacy loss for error bound minimization, All-in guarantees monotonicity by selecting data owners in ascending order of their \textit{unit valuations} $v_i^{unit}=\frac{V_i'}{\bar{\epsilon}_i'}$ until the budget is exhausted. 
% Intuitively, if a data owner $i$ increases her reported valuation $V_i'$ or decreases her reported privacy budget $\bar{\epsilon}_i’$, she stays at the same position or move to a former position in the ascending order of $v_i^{unit}$; because the financial budget $B$ still can cover the total payment $\sum_{j\in\mathcal{W}\cup \{i\}} \bar{\epsilon}_i' \cdot v_i^{unit}$ based on her unit valuation $v_i^{unit}$, she remains as a winner. On the other hand, we use the \textit{critical unit payment} $p^{unit}$ to ensure that each winner $i$' payment $p_i=\epsilon_i \cdot p^{unit}$ is critical: if she claims a higher unit valuation $v_i^{unit} > p^{unit}$, she will definitely lose the auction because of the violation of budget feasibility. Therefore, All-in is a truthful auction mechanism.

\begin{proposition}
	\label{prop:all_in}
	All-in satisfies truthfulness, IR, and BF.
\end{proposition}

\begin{proof}
	All-in satisfies IR because the critical unit payment $p^{unit}$ is no lower than each winner $i$'s unit valuation $v_i^{unit}$.
	Then, we prove that All-in satisfies truthfulness. Let $V_i'$ be the reported $V_i$, $U_i=u_i(b_i; b_{-i}',B)$ and $U_i'=u_i(b_i', b_{-i}',B)$. For each data owner $i$, we should discuss four cases as follows.
	\begin{enumerate}[leftmargin=*]
		\item $\bar{\epsilon}_i'>\bar{\epsilon}_i$ and/or $d_i > \bar{d}_i$: Obviously, data owner $i$ has no incentive because $U_i'=-\infty$.
		\item $\bar{\epsilon}_i'\!<\!\bar{\epsilon}_i$ and/or $d_i < \bar{d}_i$: In the worst case, the critical unit payment is $p^{unit'}=\frac{B}{\sum_{j\in \mathcal{W}}d_j\cdot \bar{\epsilon}_j'}$. 
		Then, we have $U_i'=d_i \cdot \bar{\epsilon}_i' \cdot p^{unit'} - V_i = \frac{B\cdot d_i \cdot \bar{\epsilon}_i'}{\sum_{j\in \mathcal{W}} d_j \cdot \bar{\epsilon}_j'} - V_i < \frac{B\cdot \bar{d}_i \cdot  \bar{\epsilon}_i}{\sum_{j\in \mathcal{W}/{i}} d_j \cdot \bar{\epsilon}_j' + \bar{d_i} \cdot \bar{\epsilon}_i} -V_i= U_i$.
		\item $\bar{\epsilon}_i'=\bar{\epsilon}_i$, $d_i = \bar{d_i}$ and $V_i'>V_i$: If $\frac{V_i'}{d_i \cdot \bar{\epsilon}_i'}$ is higher than the critical unit payment $p^{unit}=\frac{B}{\sum_{j\in \mathcal{W}/{i}} d_j \cdot \bar{\epsilon}_j' + \bar{d_i} \cdot \bar{\epsilon}_i}$, she loses the auction; otherwise, her utility does not change because the critical payment is unchanged. 
		\item $\bar{\epsilon}_i'=\bar{\epsilon}_i$, $d_i = \bar{d}_i$ and $V_i'<V_i$: Her utility does not change because of the unchanged critical payment.
	\end{enumerate}
	
\end{proof}
\end{document}